\documentclass[year=26]{fmcad}
\usepackage[T1]{fontenc}
\usepackage{booktabs}
\usepackage{tikz,tcolorbox}
\tcbuselibrary{listings,breakable}
\usepackage{graphicx}
\usepackage{multirow}
\usepackage{makecell}
\usepackage{array}
\usepackage{amsmath,amsfonts,amssymb}
\usepackage{amsthm}
\newtheorem{theorem}{Theorem}
\newtheorem{lemma}[theorem]{Lemma}

\theoremstyle{definition}
\newtheorem{definition}[theorem]{Definition}
\newtheorem{example}[theorem]{Example}
\theoremstyle{remark}

\usepackage[inline]{enumitem}
\usepackage{wrapfig}
\usepackage{subcaption}
\usepackage{mathtools}
\usepackage{etoolbox}

\DeclarePairedDelimiterX\Tuple[1]\lparen\rparen{%
  #1}

\DeclarePairedDelimiterX\Set[1]\{\}{%
#1}

\DeclarePairedDelimiterX\of[1]\lparen\rparen{%
  #1}

\DeclarePairedDelimiterX\Bof[1]\lbrack\rbrack{%
  #1}

\DeclarePairedDelimiterX\abs[1]\lvert\rvert{\ifblank{#1}{\:\cdot\:}{#1}}

\newcommand{\graph}{\mathcal{G}}
\newcommand{\graphs}{\boldsymbol{\mathcal{G}}}
\newcommand{\edge}{\mathcal{E}}
\newcommand{\nodes}{\mathcal{V}}
\newcommand{\weights}{w}
\newcommand{\timevar}{\mathbb{T}}
\newcommand{\mas}{\mathcal{MA}}

\newcommand{\sysstate}{x}
\newcommand{\type}{\mathtt{type}} 
\newcommand{\settypes}{\mathcal{T}}
\newcommand{\Aff}{\mathrm{PWA}}

\newcommand{\agentstate}{\mathbf{x}}
\newcommand{\agentinput}{\mathbf{u}}
\newcommand{\pos}{\mathbf{p}}

\newcommand{\jointstate}{\mathbf{X}}
\newcommand{\jointcontrol}{\mathbf{U}}

\newcommand{\world}{\mathbf{w}}

\newcommand{\heading}{\theta}      

\newcommand{\Inop}{\mathrm{In}}
\newcommand{\Outop}{\mathrm{Out}}

\newcommand{\Globally}{\mathbf{G}}
\newcommand{\Until}{\mathbf{U}}
\newcommand{\Eventually}{\mathbf{F}}

\newcommand{\EX}{\mathbf{EX}}
\newcommand{\FA}{\mathbf{FA}}
\newcommand{\capvec}{\boldsymbol{\kappa}}

\newcommand{\gen}{\Gamma}

\newcommand{\mypara}[1]{%
  \vspace{0.3em}%
  \noindent\textit{#1.}\
}

\usepackage{xcolor}

\newcommand{\Be}{\mathbb{B}}

\newcommand{\Ne}{\mathbb{N}}

\newcommand{\Uc}{\mathcal{U}}
\newcommand{\Vc}{\mathcal{V}}
\newcommand{\Wc}{\mathcal{W}}
\newcommand{\Xc}{\mathcal{X}}

\newcommand{\worldmodel}{f} 

\newenvironment{runningex}[1]{%
  \par\addvspace{\topsep}%
  \trivlist\item[\hskip\labelsep\itshape
    Example~\ref{#1} (continued).]%
  \rmfamily
}{%
  \endtrivlist
}

\begin{document}

\usetikzlibrary{automata,positioning,arrows}
\newtcolorbox{jcomment}{
  colback=blue!10,
  colframe=blue!40!black,
  boxrule=0.5pt,
  arc=3pt,
  left=4pt,
  right=4pt,
  top=2pt,
  bottom=2pt,
  breakable
}

\title{Multi-Agent Planning with Spatio-Temporal and Topological Constraints using STL-GO}

\author{
  \IEEEauthorblockN{
    Sheryl Paul$^{1,*}$,\;
    Vidisha Kudalkar$^{1,*}$,\;
    Anand Balakrishnan$^{2}$,\;\\
    Lars Lindemann$^{3}$,\;
    Alberto Speranzon$^{4}$,\;
    Jyotirmoy V. Deshmukh$^{1}$
  }
  \IEEEauthorblockA{
    $^{1}$University of Southern California,
    Los Angeles, CA, USA \;
    \{sherylpa, kudalkar, jdeshmuk\}@usc.edu\\
    $^{2}$University of Texas at Austin,
    Austin, TX, USA \;
    anandbal@utexas.edu\\
    $^{3}$ETH Zurich,
    Zurich, Switzerland \;
    llindemann@ethz.ch\\
    $^{4}$Lockheed Martin,
    Eagan, MN, USA \;
    alberto.speranzon@lmco.com\\
    $^{*}$Equal contribution.
  }
}
\maketitle

\begin{abstract}
Multi-agent planning problems arise in a variety of engineering
applications, such as multi-robot wildfire fighting and unmanned aerial
inspection in factories.
A particular challenge is the existence of spatio-temporal (i.e., when and/or where an
agent should do what) and topological constraints (i.e., how agents should
interact), as typically formalized via the notion of graphs.
Over the last years, various frameworks have been proposed that can capture such
constraints via spatio-temporal logics.
We focus here on \emph{spatio-temporal logic with graph operators
(STL-GO)}, a recent formalism that supports reasoning about multiple agents
and their topologies, such as sensing, communication, and task topologies.
In this paper, we consider the problem of planning multi-agent paths that
satisfy constraints written in STL-GO.
This problem is particularly challenging due to the need of encoding
multiple, potentially time-varying graphs via the graph operators inherent
to STL-GO.
We present two encodings of this problem, one based on mixed-integer
programming (MIP) and another based on satisfiability modulo theory (SMT), with soundness guarantees.
We provide a unified interface for specifying agent constraints, their graph
topologies, and the STL-GO specification, enabling seamless use of both
methods and facilitating direct comparison between them.
We evaluate both encodings on a multi-UAV
search-and-rescue benchmark, ablating over team size and graph
complexity,
highlighting the expressiveness of the proposed encodings under dynamic
multi-graph interactions.
\end{abstract}


\section{Introduction}
Classical temporal logic specifications, such as LTL or STL, are well suited for
expressing time-based properties of individual system trajectories. However, in
multi-agent systems (MAS), desired mission behavior may depend not only on when
events occur but also on inter-agent communication, spatial relationships, and
task dependencies. Furthermore, mission objectives may need to reason about how
these inter-agent relations evolve over time. A useful model is thus to treat a
multi-agent system as a collection of directed or undirected graphs, where nodes
represent agents with dynamic behavior and time-varying edges across multiple
graphs capture distinct kinds of inter-agent relationships.
\begin{figure}[t]
    \setlength{\belowcaptionskip}{-1.33em}
    \centering
    \includegraphics[width=0.8\columnwidth]{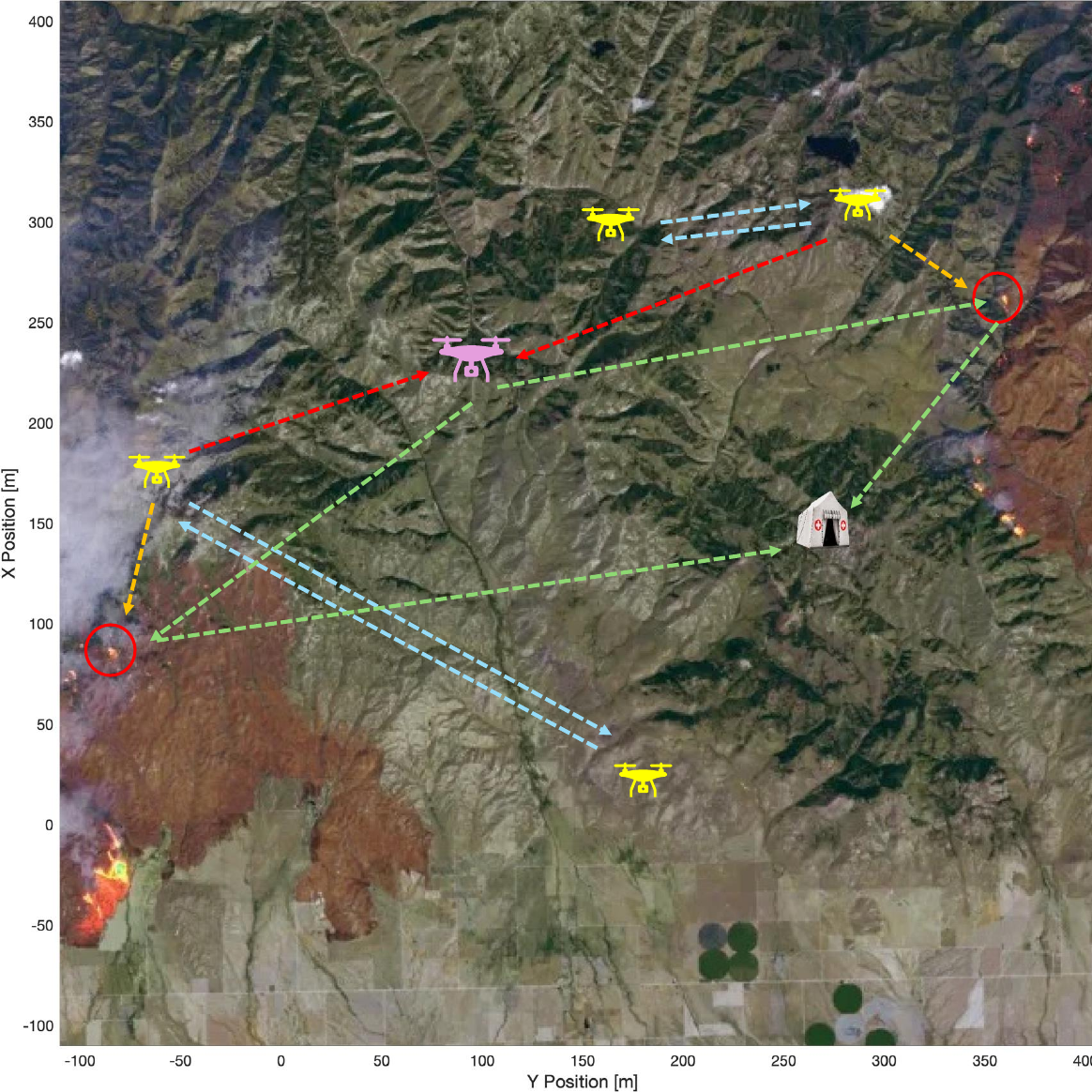}
    \caption{\footnotesize
        Motivation scenario: a heterogeneous multi-agent system coordinating
        wildfire response over a satellite terrain map.
        \textbf{Yellow drones} are \emph{locator} agents that patrol the region
        to monitor fire spread and detect emergencies (marked by red circles).
        The \textbf{purple drone} is the \emph{rescuer} agent tasked with
        reaching survivors and transporting them to the rescue center
        (white tent).
        \textcolor{orange}{\textbf{Orange arrows}} indicate
        \emph{sensing}: a locator drone detects an emergency site.
        \textcolor{cyan}{\textbf{Blue arrows}} represent
        \emph{inter-agent communication} links through which locators share
        situational awareness.
        \textcolor{red}{\textbf{Red arrows}} denote
        \emph{task assignment}: the detected emergency is assigned to the
        rescuer.
        \textcolor{green!60!black}{\textbf{Green arrows}} trace the
        \emph{rescue path}: the rescuer navigates first to the emergency
        location, then to the rescue center.
        Our goal is to synthesize agent trajectories that satisfy a
        \emph{spatio-temporal reachability specification} encoding these
        coordination requirements with formal correctness guarantees.
    }
    \label{fig:motivation}
\end{figure}

Recent work has focused on spatio-temporal logic formalisms such as SSTL
\cite{bortolussi2014specifying,nenzi2015qualitative}, SaSTL \cite{ma2020sastl},
SpaTeL \cite{haghighi2015spatel}, STREL~\cite{STREL,STRELDynamicNetworks},
Census STL \cite{xu2016census} and STL-GO \cite{stlgo}. Among these formalisms,
STL-GO allows specifying multi-agent behaviors over multiple inter-agent
relational structures: it supports simultaneous quantification over distinct
time-varying agent relationships. As a motivational example
(Fig.~\ref{fig:motivation}), consider a wildfire-response setting in which:
``\emph{Every emergency situation must be sensed by a locator agent within
bounded time; upon sensing, the locator must communicate the emergency to
connected agents and upon contact with a rescuer, assign that rescuer to the
emergency. The rescuer must then reach the emergency location and perform the
rescue action and return to a safe location.}'' This specification reasons over
three distinct topologies at once (sensing, communication, and task-assignment),
and also constrains the number of neighbors satisfying a desired property. Such
types of specifications are a first step to formalizing the behavior of
practical multi-agent systems such as those used for search-and-rescue,
environmental monitoring under limited communication, and distributed sensing in
uncertain terrains.

STREL provides path-based reachability and escape operators over a chosen
dynamic weighted spatial model, while Census STL counts agents within a
population; neither directly combines STL-GO's typed neighborhood-cardinality
operators with explicit quantification over collections of interaction graphs.
Hyperproperty logics such as
HyperLTL~\cite{hsu2025hyprl,wang2020hyperproperties,finkbeiner2023logics} can
also express relational properties across agents, but require lifting
agent-level quantifiers outside any temporal operator. 
Consequently,
temporal changes in agent sets and interaction relations must be represented
through explicit propositions and finite-domain expansion. We use HyperLTL as
a relational specification comparison grounded in prior planning work, rather
than as a general-purpose baseline for multi-agent systems in Section~\ref{sec:conc}.

\noindent In prior work \cite{stlgo}, the authors focus on runtime monitoring of STL-GO
specifications, assuming that agents' spatio-temporal behaviors are decided
by some given planner; how such plans can be synthesized is not addressed.
The main problem we consider in this paper is open-loop, centralized, bounded-horizon
planning for a multi-agent system subject to STL-GO specifications where the agent and environment dynamics are deterministic and known. We
focus on this case as a foundational baseline: a tractable
centralized encoding is a prerequisite for decentralized extensions and,
to our knowledge, no such encoding exists for STL-GO.

\noindent There has been substantial work on planning multi-agent systems subject to temporal logic specifications such as LTL \cite{OnlineMultiRobotLTL,SMTMultiRobotSafeLTL}, STL~\cite{FormalMethodsMultiAgent,MultiAgentSTLWaypoints}, ATL~\cite{ATL} and CaTL~\cite{CaTLPlus}. Existing work can be broadly categorized as follows: \emph{constraint-solving/SMT-based synthesis}, which encodes temporal logic constraints symbolically and reasons about feasibility or correctness using SAT/SMT solvers \cite{shoukry2016scalable,shoukry2017linear}; \emph{Mixed-Integer Programming (MIP)} approaches, which formulate motion planning and control synthesis under temporal logic constraints as optimization problems over continuous dynamics and binary decision variables \cite{SMTMultiRobotSafeLTL,MultiAgentSTLWaypoints}; \emph{reactive synthesis}, which focuses on strategy synthesis and correctness guarantees in adversarial or game-theoretic settings, commonly using ATL or related formalisms \cite{ATL}; and \emph{learning-based solvers}~\cite{NNSTREL,formats}.

Inspired by prior work on SMT- and MIP-based planning, we present SMT and MIP
encodings of the centralized planning problem under STL-GO specifications. A key
difference from prior encoding methods is that we explicitly handle weighted,
time-varying interaction graphs whose structure is induced by the joint state of
agents and environment; this requires encoding multi-graph quantification and
neighborhood-cardinality predicates as solver constraints.

\noindent \textit{Contributions.}
(i) We present MIP and SMT encodings of STL-GO that support multi-graph
    existential and universal quantification and neighborhood-cardinality
    predicates over time-varying interaction graphs, with soundness
    guarantees. 
    (ii) We provide a unified interface for specifying agents, interaction graphs, STL-GO formulas, and an objective function that compiles to MIP/SMT-encoded plans, also enabling direct
    empirical comparison.
(iii) We empirically evaluate the encodings on a multi-UAV search-and-rescue benchmark
    using Gurobi and Z3, ablating over team size and interaction-graph complexity.
   (iv) We additionally evaluate the encodings on a structurally
different grid-world benchmark adapted from HypRL
\cite{hsu2025hyprl}. The results compare solve time and encoding size across
the MIP and SMT backends.
    We further show that the
    same specifications, encoded in HyperLTL, incur either a linear
    blow-up in formula size or an alternation in the quantifier prefix,
    motivating STL-GO's pointwise agent-level quantification.

The remainder of the paper is organized as follows. Section~\ref{sec:preliminaries} introduces the multi-agent system model, interaction graphs, and STL-GO. Section~\ref{sec:problem-statement} formalizes the bounded-horizon planning problem for multi-agent systems subject to STL-GO specifications. Sections~\ref{sec:stlgo-milp} and~\ref{sec:stlgo-smt} present our solver-based synthesis approaches, describing the translation of STL-GO specifications into MIP and SMT encodings, respectively. We evaluate both approaches through simulations and present our results in Section~\ref{sec:exp}, and discuss related work and conclusions in Section~\ref{sec:conc}.

\section{Preliminaries}
\label{sec:preliminaries}


\subsection{Multi-Agent System Model}

Let \(\nodes = \{1,\ldots,N\}\) be a set of agents whose spatio-temporal
behavior evolves over a discrete time domain \(\timevar \subset \Ne\). We
consider a homogeneous MAS in which each agent \(i \in \nodes\) has a state
vector \(\agentstate^i_t \in \Xc\), at time $t$ where \(\Xc\) is the shared state space. The
state variables encode attributes local to an agent, such as physical
configuration and internal resources.
The joint state of all agents at time \(t \in \timevar\) is denoted
\(\jointstate_t = (\agentstate_t^1, \ldots, \agentstate_t^N) \in \Xc^N\). While we focus on the homogeneous setting for
notational clarity, the encoding extends to role-typed heterogeneity, as
used in the search-and-rescue example below.

The MAS environment is modeled as a discrete-time dynamical system with state
\(\world_t \in \Wc\), where \(\Wc\) encodes attributes such as map geometry
(static or dynamic), obstacles, and adversarial features (e.g., communication
disruption). The environment evolves deterministically according to \(\world_{t+1} =
\worldmodel(\world_t)\), where \(\worldmodel\) and the initial world state
\(\world_0\) are known. Thus, the finite trajectory
\(\{\world_t\}_{t=0}^{T}\) is fixed when a planning instance is constructed.\footnote{Real-world environments may exhibit
stochasticity, in which case a synthesized plan can only be guaranteed to
satisfy the specification with high probability. We do not address this setting
here, but the tools developed in this paper can serve as a building block for
stochastic extensions.}

Each agent is governed by a transition dynamics function \(F\) that is shared
across all agents (since the MAS is homogeneous).
At each time \(t\), agent \(i\) selects a control input \(\agentinput_t^i \in
\Uc\) from an admissible input domain \(\Uc\), and its successor state is
\begin{equation}
  \agentstate_{t+1}^i = F(\agentstate_t^i, \agentinput_t^i, \world_t).
\end{equation}
Stacking the componentwise dynamics gives the joint successor state
\begin{equation}
  \jointstate_{t+1} = F(\jointstate_t, \jointcontrol_t, \world_t),
  \label{eq:agent-dyn}
\end{equation}
where \(\jointcontrol_t = (\agentinput_t^1, \ldots, \agentinput_t^N)\) is the joint control input.

\begin{example}
  \label{example:running}
  Consider a system of \(\Vc = \left\{ 1, \ldots, N \right\}\) robots, where the state of agent \(i \in \Vc\) is
  $\agentstate^i_t := \left(\pos^i_t, \heading^i_t, \capvec^i_t\right)$
  where: $\pos^i_t = [x^i_t,\,y^i_t]^\top \in \mathbb{R}^2$ is the 
  position of the agent, $\heading^i_t \in [0,2\pi)$ is the orientation, and $\capvec^i_t$ can refer to a collection of time-varying resources and capabilities such as the battery level.  
  At each time \(t\), each agent selects $u^i_t := (v^i_t, \omega^i_t) \in \Uc$ with $v^i_t \in
[v_{\min}, v_{\max}]$ and $\omega^i_t \in [\omega_{\min}, \omega_{\max}]$
the linear and angular velocities. For sampling time $\Delta t > 0$ and a
unicycle model,\footnote{The unicycle dynamics in this example are illustrative. The MIP and SMT encodings and the reported experiments use discrete-time affine motion dynamics. Accommodating the displayed nonlinear model would instead require a piecewise-affine approximation or an SMT theory of nonlinear real arithmetic.} position and orientation update as
$x^i_{t+1} = x^i_t + \Delta t\, v^i_t \cos(\theta^i_t)$,
$y^i_{t+1} = y^i_t + \Delta t\, v^i_t \sin(\theta^i_t)$,
$\theta^i_{t+1} = \mathrm{wrap}_{[0,2\pi)}(\theta^i_t + \Delta t\, \omega^i_t)$.
\end{example}


To encode the rich interaction and potential coupling between each agent,
its perception of the world state, and their effect on decision making in
the system, we define the notion of \emph{interaction graphs}
\cite{stlgo}.

\begin{definition}[Interaction Graph]
  An interaction graph \(\graph^{\type}_t\) is a directed and weighted
  graph
  \(
  \graph^{\type}_t := (\nodes,\edge^{\type}_t,\weights^{\type}_t),
  \)
  where $\edge_t^{\type}\subseteq\nodes\times\nodes$ is the edge set, and
$\weights_t^{\type}:\edge_t^{\type}\to\mathbb{R}_{\ge 0}$ assigns edge attributes
  (e.g., distance, cost, signal quality).
  Different interaction modalities are modeled by distinct graph types \(\type \in
  \settypes := \left\{ \type_1, \ldots, \type_M \right\}\).
  The collection of all interaction graphs at time $t$ is
  \(\graphs_t := \{\graph_t^{\type} \mid \type\in\settypes\}\).
\end{definition}


\begin{example}
Examples of interaction graphs include:
(i) A \textit{distance graph} \(\graph_t^d = \left(
          \nodes,\edge_t^d, w_t^d \right)\) is a complete directed graph where
          \(\edge_t^d = \nodes \times \nodes \setminus \{(i,i) \mid i \in \Vc\}\)
          (all ordered pairs where \(i \neq j\)) and \(w_t^d\) is the distance
          between agents \(i\) and \(j\) at time \(t\). (ii) A \textit{sensing graph}
           $\graph^{s}_t=(\nodes,\edge^s_t,\weights^s_t)$ encodes whether agent
           $i$ can sense agent $j$ at time $t$ and $(i,j)\in\edge^s_t$. (iii) A \textit{communication graph}  $\graph_t^c = (\nodes,\edge^c_t,\weights^c_t)$ where $(i,j)\in\edge^c_t$ if and only if agent $i$ can communicate with agent $j$. (iv) A \textit{task-dependency graph} 
\(
\graph^{\mathrm{task}}_t = (\nodes,\edge^{\mathrm{task}}_t,\weights^{\mathrm{task}}_t),
\)
indicates that agent $j$ depends on agent $i$ for task execution, where \((i, j) \in \edge^{\mathrm{task}}_t\).
\end{example}

\subsection{Spatio-Temporal Logic with Graph Operators (STL-GO)}

STL-GO \cite{stlgo} extends Signal Temporal Logic (STL) \cite{stl-dejan} with graph operators to enable reasoning about both spatio-temporal and \emph{topological} relationships between agents.
The syntax and semantics of STL-GO are split into agent-local formulas and multi-agent compositional formulas.

\subsubsection{Agent-Local Formulas}
For formulas that specify behavior local to the perspective of a single agent, we use the recursive syntax: 
\begin{displaymath}
    \varphi ::= \top \mid \mu_x \mid \neg \varphi \mid 
\varphi \wedge \varphi \mid 
\varphi \, \Until_I \, \varphi \mid 
\Inop_{\graphs,E}^{W,\#}\varphi \mid
\Outop_{\graphs,E}^{W,\#}\varphi.
\end{displaymath}
Here, \(\mu_x\) represents an atomic predicate of the form
\(\mu_x:\Xc\to\Be\) that maps the state of an agent to a Boolean value;
logical negation \(\neg\varphi\) and logical conjunction
\(\varphi\land\varphi\) are defined as usual; and \(\Until_I\) is the until
operator over an interval \(I=[a,b]\) as defined in STL
\cite{stl-dejan}.\footnote{We also use the standard derivations
\(\varphi_1\lor\varphi_2:=\neg(\neg\varphi_1\land\neg\varphi_2)\),
\(\varphi_1\Rightarrow\varphi_2:=\neg\varphi_1\lor\varphi_2\),
\(\Eventually_I\varphi:=\top\,\Until_I\,\varphi\), and
\(\Globally_I\varphi:=\neg\Eventually_I\neg\varphi\).}

STL-GO introduces the \emph{incoming}
\(\Inop_{\graphs,E}^{W,\#}\) and \emph{outgoing}
\(\Outop_{\graphs,E}^{W,\#}\) graph operators, where
\(W=[w_1,w_2]\subseteq\mathbb{R}\) constrains the edge weights. The
cardinality constraint $E$ is specified by endpoints
\(e_1\in\mathbb{N}\), \(e_2\in\mathbb{N}\cup\{\infty\}\), with
\(e_1\leq e_2\), and is defined as
\(E=[e_1,e_2]_{\mathbb N}:=\{k\in\mathbb{N}\mid e_1\leq k\leq e_2\}\).
Thus, \(E=[1,\infty)_{\mathbb N}\), for example, requires at least one
qualifying edge. Finally, \(\#\in\{\exists,\forall\}\) denotes existential
or universal quantification over the graph types in \(\settypes\), equivalently
over the graph instances in \(\graphs_t\).

Let \(\mas\) denote the finite execution induced by the multi-agent system,
the environment, and the chosen control sequence. We write
\((\mas,i,t)\models\varphi\) to denote Boolean satisfaction of the
agent-local formula \(\varphi\) at agent \(i\) and time \(t\).

Graph operators allow an agent to reason about the trajectories of its
neighboring agents. The incoming operator
\(\Inop_{\graphs,E}^{W,\exists}\varphi\) asserts that there exists at least
one graph instance \(\graph_t^{\type}\in\graphs_t\) for which the count of
incoming edges \((j,i)\) to agent \(i\) satisfying both
\(\weights_t^{\type}(j,i)\in W\) and
\((\mas,j,t)\models\varphi\) lies in \(E\). The universal version
\(\Inop_{\graphs,E}^{W,\forall}\varphi\) requires the same property to hold
across all graph instances \(\graph_t^{\type}\in\graphs_t\). Similarly,
the outgoing operator \(\Outop_{\graphs,E}^{W,\exists}\varphi\) states that
there exists a graph instance for which the analogous count over outgoing
edges \((i,j)\) from agent \(i\) lies in \(E\). If weights are not of
interest, we set \(W:=(-\infty,\infty)\) and write the simplified forms
\(\Inop_{\graphs,E}^{\#}\varphi\) and
\(\Outop_{\graphs,E}^{\#}\varphi\).

Formally, writing \(t\oplus I:=\{t+\tau:\tau\in I\}\), we define the
recursive semantics below.\footnote{We use strong bounded-horizon semantics:
if \(t+b>T\), an until or eventually obligation over \(I=[a,b]\) cannot be
discharged within the encoded horizon and is therefore false. Because
\(\Globally_I\varphi:=\neg\Eventually_I\neg\varphi\), the derived globally
formula is true in this boundary case.}

\begin{displaymath}
\scalebox{0.91}{$
\begin{array}{l@{\;}l}
(\mas,i,t)\models\top
&\text{always},\\
(\mas,i,t)\models\mu_x
&\text{iff }\mu_x(\agentstate_t^i),\\
(\mas,i,t)\models\neg\varphi
&\text{iff }(\mas,i,t)\not\models\varphi,\\
(\mas,i,t)\models\varphi_1\land\varphi_2
&\text{iff }(\mas,i,t)\models\varphi_1\\
&\quad\land(\mas,i,t)\models\varphi_2,\\
(\mas,i,t)\models\varphi_1\,\Until_{[a,b]}\,\varphi_2
&\text{iff }t+b\leq T, \exists t'\in t\oplus[a,b]\text{ s.t. }\\
&\quad(\mas,i,t')\models\varphi_2\\
&\quad\land\ \forall t''\in[t,t').\,
(\mas,i,t'')\models\varphi_1,\\
(\mas,i,t)\models\Inop_{\graphs,E}^{W,\#}\varphi
&\text{iff }
\#\,\type\in\settypes\ \text{s.t.}\\
&\quad\left|
\left\{
(j,i)\in\edge_t^{\type}:
\weights_t^{\type}(j,i)\in W,\right.\right.\\
&\quad\left.\left.
(\mas,j,t)\models\varphi
\right\}
\right|\in E,\\
(\mas,i,t)\models\Outop_{\graphs,E}^{W,\#}\varphi
&\text{iff }
\#\,\type\in\settypes\ \text{s.t.}\\
&\quad\left|
\left\{
(i,j)\in\edge_t^{\type}:
\weights_t^{\type}(i,j)\in W,\right.\right.\\
&\quad\left.\left.
(\mas,j,t)\models\varphi
\right\}
\right|\in E.
\end{array}
$}
\end{displaymath}
\subsubsection{Multi-Agent Formulas}

STL-GO uses the following recursive grammar to define properties over multiple agents:
$$
\phi ::= \top \mid \mu \mid i.\varphi \mid \neg \phi \mid \phi \wedge \phi
\mid \phi \Until_I \phi
\mid \FA\,\varphi \mid \EX\,\varphi,
$$
where $\varphi$ is an agent-local formula, and \(\mu\) (without the subscript
\(x\)) denotes an atomic predicate of the form
\(\mu:\Xc^{\abs{\nodes}}\times\Wc\to\Be\). Such predicates are analogous to
the agent-local predicates \(\mu_x\), but are defined over the joint agent and
world state \((\jointstate_t,\world_t)\).

Multi-agent formulas allow properties to be specified across agents using
logical connectives and over time using temporal operators. The operator
\(i.\varphi\) embeds an agent-local formula into a multi-agent formula. The
operators \(\FA\) and \(\EX\) denote universal and existential
quantification, respectively, over the full agent set \(\nodes\), with
\(\FA\varphi:=\bigwedge_{i\in\nodes}i.\varphi\) and
\(\EX\varphi:=\bigvee_{i\in\nodes}i.\varphi\).

Formally, we write \((\mas,t)\models\phi\) to denote that the STL-GO formula
\(\phi\) is satisfied by the finite execution \(\mas\) at time \(t\). Its
semantics are defined inductively as follows:
\begin{displaymath}
\begin{array}{l@{ }l}
(\mas,t)\models\top
&\text{always},\\
(\mas,t)\models\mu
&\text{iff }\mu(\jointstate_t,\world_t),\\
(\mas,t)\models i.\varphi
&\text{iff }(\mas,i,t)\models\varphi,\\
(\mas,t)\models\neg\phi
&\text{iff }(\mas,t)\not\models\phi,\\
(\mas,t)\models\phi_1\land\phi_2
&\text{iff }(\mas,t)\models\phi_1
\!\land\!(\mas,t)\!\models\!\phi_2,\\
(\mas,t)\models\phi_1\,\Until_{[a,b]}\,\phi_2 \,
&\text{iff }t+b\leq T\text{ and }\exists t'\in t\oplus[a,b]\\
&\quad\text{s.t. }
(\mas,t')\models\phi_2\\
&\quad\land\ \forall t''\in[t,t').\,
(\mas,t'')\models\phi_1,\\
(\mas,t)\models\FA\varphi
&\text{iff }\forall i\in\nodes.\,(\mas,i,t)\models\varphi,\\
(\mas,t)\models\EX\varphi
&\text{iff }\exists i\in\nodes.\,(\mas,i,t)\models\varphi.
\end{array}
\end{displaymath}
\begin{runningex}{example:running}
In our example, let us consider the UAVs to be assigned a search-and-rescue mission with agents being assigned one of two roles: `locators' $\mathcal{L} \subseteq \nodes$ and `rescuers' $\mathcal{R} \subseteq \nodes$.

Emergency events may arise at any site in a fixed set of candidate locations and must be detected by the locator
agents and assigned to one or more rescuers. The rescuer must resolve the emergency within bounded time, i.e., reach the emergency location and carry the rescued individual to the rescue center.
Let $\mathcal{C}\subset\mathbb{R}^2$ denote the rescue center and $\mathcal{M}$ a fixed set of possible emergency sites, with each $m \in \mathcal{M}$ having position $\pos^m_t \in \mathbb{R}^2$. Let $\varepsilon_E,\varepsilon_C>0$ be distance tolerances for reaching an emergency and the rescue center, respectively.
For $\ell\in\mathcal{L}$, $r\in\mathcal{R}$, and $m\in\mathcal{M}$, we define the
atomic predicates 
$\varphi^{\mathrm{emg}}_{m}(t) \iff \text{emergency $m$ is active at $t$}$,
$\varphi^{\mathrm{near}}_{r,m}(t) \iff \|\pos^r_t-\pos^m_t\|\le \varepsilon_E$,
$\varphi^{\mathrm{atC}}_r(t)
\iff \mathrm{dist}(\pos^r_t,\mathcal{C})\le \varepsilon_C$,
$\varphi^{\mathrm{carry}}_r(t)
\iff \text{rescuer $r$ carries a rescued individual}$.
Let $\varphi^{\mathrm{task}}_{\ell,r,m}(t)$ denote that locator $\ell$
assigns rescuer $r$ specifically to emergency $m$ at time $t$.
Let $E_{\geq1}:=[1,\infty)_{\mathbb N}$. The emergency-sensing predicate is a joint geometric predicate, while the
remaining predicates below are agent-local predicates derived using the STL-GO
graph operators:\footnote{For the communication predicates, the experimental implementation uses role-partitioned restrictions of the communication graph: $\varphi^{\mathrm{LL}}_\ell$ restricts target neighbors to locators, whereas $\varphi^{\mathrm{LR}}_\ell$ restricts target neighbors to rescuers. The compact notation below suppresses these target-role restrictions.}
\begin{align*}
  \varphi^{\mathrm{sense}}_m(t)    &:=
    \bigvee_{\ell\in\mathcal L}
    \left(\|\pos_t^\ell-\pos_t^m\|\le r_{\mathrm{sense}}\right)\\
  \varphi^{\mathrm{LL}}_\ell(t)    &:= \Outop_{\{\graph^{c}\},E_{\geq1}}^{\exists}\,\top
    \\
  \varphi^{\mathrm{LR}}_\ell(t)    &:= \Outop_{\{\graph^{c}\},E_{\geq1}}^{\exists}\,\top
\end{align*}
We encode the following specification: Once an emergency is detected, at least one locator must assign a rescuer within
        bounded time $T_{\mathrm{assign}}$:
 \[
 \resizebox{\linewidth}{!}{$\displaystyle
          \phi_{\mathrm{assign}}
          := \Globally
          \bigwedge_{m\in\mathcal{M}}
          \left(
          \varphi^{\mathrm{emg}}_m \wedge \varphi^{\mathrm{sense}}_m
          \Rightarrow
          \Eventually_{[0,T_{\mathrm{assign}}]}
          \bigvee_{\ell\in\mathcal L}\;
          \bigvee_{r\in\mathcal R}
          \varphi^{\mathrm{task}}_{\ell,r,m}
          \right).
 $}
	   \]
\label{ex:running}
\end{runningex}

\section{Problem Statement}
\label{sec:problem-statement}
In this paper, we are interested in an \emph{open-loop planning}\footnote{Here, ``open-loop'' refers to the finite control sequence synthesized along each realized execution branch. In the experiments, multiple such sequences are synthesized jointly to form a finite contingency tree; observations select among precomputed branches, while the selected branch itself is executed open loop.} or
\emph{bounded synthesis} problem for multi-agent systems, that is, our goal
is to synthesize a finite sequence of control inputs to a system of multiple
homogeneous agents over a bounded horizon, such that they satisfy a formal
specification.

More concretely, we consider a system of agents
\(\nodes=\Set*{1,\ldots,N}\) with state space \(\Xc\) and joint input
\(\jointcontrol_t
=(\agentinput_t^1,\dots,\agentinput_t^N)\in\Uc^N\), operating in a world
with state space \(\Wc\).
Each agent evolves in discrete time according to a known, deterministic,
homogeneous dynamics function
\(F:\Xc\times\Uc\times\Wc\to\Xc\). For the encodings presented in this
paper, we restrict \(F\) to be affine in the agent state, control input, and
world state.
We assume that the initial joint state \(\jointstate_0\) and the finite world
trajectory \(\{\world_t\}_{t=0}^{T}\) are fixed and known when the planning
instance is constructed.

Let $\phi$ be an STL-GO formula interpreted over the MAS $\mas$ under the
induced graph-collection sequence $\{\graphs_t\}_{t=0}^T$ over a bounded horizon \(T \in \Ne\).
Then, our goal is to synthesize an open-loop control sequence
$\{\jointcontrol_t\}_{t=0}^{T-1}$ such that the resulting execution
$\{\jointstate_t\}_{t=0}^{T}$ according to the dynamics in
\autoref{eq:agent-dyn} satisfies the specification at the initial state, i.e., $ (\mas,0)\models \phi$.
One may optionally minimize a performance objective
\(J(\{\jointstate_t\}_{t=0}^{T},
\{\jointcontrol_t\}_{t=0}^{T-1})\) (e.g., control effort, total
path length, etc.), yielding an \emph{optimal} bounded synthesis problem.

We identify a class of multi-agent systems
with state-dependent interaction graphs characterized by \emph{graph constructor functions}, for which the bounded
synthesis can be framed as an optimization-based or satisfaction problem
without much change in how the problem itself is encoded.
\begin{definition}[Graph Constructor Function]
\label{def:graph-constructor}
For a given interaction modality \(\type\in\settypes\), a graph constructor
function \(\gen^{\type}\) maps an element of
\(\Xc^{|\nodes|}\times\Wc\) to a directed, weighted graph on \(\nodes\).
At time \(t\), we write
\(\graph_t^{\type}
=\gen^{\type}(\jointstate_t,\world_t)
=(\nodes,\edge_t^{\type},\weights_t^{\type})\).
Thus, for every ordered pair \((i,j)\), the constructor determines whether
\((i,j)\in\edge_t^{\type}\) and, when the edge exists, its weight
\(\weights_t^{\type}(i,j)\).
\end{definition}
\begin{example}
     Let $\type = d$ denote a distance-based interaction modality. $\Gamma^{d}(\jointstate_t,\world_t)$ maps to the
function $(i,j) \mapsto d_1(\agentstate^i_t, \agentstate^j_t)$, where $\agentstate^i_t$
and $\agentstate^j_t$ denote the 2D (or 3D) coordinates of agents $i$ and $j$ and $d_1$ is the $\ell_1$ distance.\footnote{$d_1(a, b) = \sum_{k=1}^{n} |a^k - b^k|$ where $k$ indexes the $n$ spatial dimensions.}
\end{example}

Under the determinism and homogeneity assumptions above, and restricting
$\gen^\type$ to an encodable class (made precise in Sections~\ref{sec:stlgo-milp} and~\ref{sec:stlgo-smt}), the bounded synthesis problem can be
transformed into a finite set of constraints suitable for satisfiability
checking (e.g. SMT), or optimization (e.g. MIP), by introducing:
\begin{enumerate*}[label={(\roman*)}]
  \item decision variables for
        \(\{\agentstate_t^i\}_{t=0}^{T}\) and
        \(\{\agentinput_t^i\}_{t=0}^{T-1}\),
  \item auxiliary binary variables encoding the truth of various subformulas
        of $\phi$ for each agent and time instant
  \item constraints enforcing the dynamics, graph construction predicates,
        and the STL-GO semantics of Boolean, temporal, and graph operators
        over the bounded horizon.
\end{enumerate*}
While we focus on homogeneous agent dynamics, the encoding extends to role-typed heterogeneity, as used in the locator/rescuer split of the running example.

\mypara{Assumptions} We focus on \emph{centralized}, \emph{open-loop} planning under
\emph{deterministic} environment dynamics. Decentralized planning under partial
observability and local observation is the ultimate target, but we consider the centralized, fully observable setting as a foundational step; a tractable encoding in this setting will serve as a prerequisite for decentralized extensions which we relegate to future work. Additionally, we assume that agents execute the
synthesized open-loop plan \emph{autonomously}, and the deterministic environment
dynamics let us reason about feasibility and correctness without
probabilistic semantics.

In the following sections we present a systematic modeling and synthesis
framework that compiles the planning problem above, with its multiple interaction graphs, into mixed-integer programs or satisfiability problems for centralized planning.
Due to space constraints, we will introduce the novel graph-related encodings, and refer the reader to the complete encodings described in the
Appendices~\ref{sec:appendix-MILP-encoding} and~\ref{sec:appendix-SMT-encoding}.

\section{MIP-based Encoding for STL-GO Planning}
\label{sec:stlgo-milp}

\subsection{Encoding System Specifications}
\label{sec:milp-encodable-graphs}

\noindent \mypara{System Dynamics}
We restrict ourselves to deterministic, homogeneous dynamics that are affine in the agent
state, agent input, and world state so that the bounded synthesis problem admits
an MIP encoding.
Specifically, let the dynamics be
\begin{align}
\label{eq:affine-agent-dyn}
\agentstate^{i}_{t+1}
=
F(\agentstate^{i}_{t},\agentinput^{i}_{t},\world_t)
:=
    A \agentstate^{i}_{t} + B \agentinput^{i}_{t} + E \world_t + c,
\end{align}
where \( i\in\nodes,\ \ t=0,\ldots,T-1, \)
and $A,B,E$ and $c$ have appropriate dimensions. 

\mypara{State and input constraints} We assume agent states and control inputs to lie within hyper-rectangular sets, so that they are bounded componentwise as:
\begin{equation}
\label{eq:state_input_constraints}
x_{\min} \leq \agentstate^i_t \leq x_{\max},  \text{  and  } u_{\min} \leq \agentinput^i_t \leq u_{\max}, \quad \forall i, t.
\end{equation}

\noindent \mypara{Interaction Graphs} At each time $t$, interaction graphs are constructed from the joint agent state
$\jointstate_t\in\mathcal{X}^N$ and world state $\world_t\in\mathcal{W}$ via the graph constructor function $\gen^{\type}$ of Definition~\ref{def:graph-constructor}.
For each ordered pair $(i,j)\in\nodes\times\nodes$, $i\neq j$, the existence of a
directed edge $(i,j)$ is represented by a Boolean variable, while edge weights
are real-valued. Specifically, write
\(
\eta_{i,j}^{\type}:\mathcal{X}^{|\nodes|}\times\mathcal{W}\to\Be
\)
for the edge-existence predicate and
\(
e_{i,j}^{\type}:\mathcal{X}^{|\nodes|}\times\mathcal{W}\to\mathbb{R}_{\geq0}
\)
for the edge-weight expression. Then
\((i,j)\in\edge_t^{\type}\) if and only if
\(\eta_{i,j}^{\type}(\jointstate_t,\world_t)\) holds, and an existing edge
has weight
\(\weights_t^{\type}(i,j)=e_{i,j}^{\type}(\jointstate_t,\world_t)\).

\noindent To ensure compatibility with solver-based synthesis, we restrict
attention to graph constructors whose edge-existence predicates admit exact
mixed-integer encodings and whose edge-weight functions are piecewise affine
in $(\jointstate_t,\world_t)$.
Let $\Aff(\mathcal{X}^{|\nodes|} \times \mathcal{W})$ denote the set of 
piecewise-affine functions over the joint agent and world state.
We say that a graph constructor $\gen^{\type}$ is MIP-encodable if, for every
ordered pair $(i,j)$, $\eta_{i,j}^{\type}$ is a Boolean combination of affine
comparisons with an exact mixed-integer representation and
$e_{i,j}^{\type}\in\Aff(\mathcal{X}^{|\nodes|}\times\mathcal{W})$.
The encoding introduces $a_{i,j,t}^{\type}\in\{0,1\}$ with
$a_{i,j,t}^{\type}=1$ if and only if
$\eta_{i,j}^{\type}(\jointstate_t,\world_t)$ holds. Atomic affine comparisons
are encoded using two-sided Big-$M$ constraints with valid bounds and the
required numerical separation margins; Boolean connectives are translated
into linear constraints over binary variables.

\subsection{Encoding Agent-Local Operators}
\label{sec:milp-agent-local}
\mypara{Logical and Temporal Operators} The encodings of atomic predicates, logical operators, and temporal operators
closely follow established formulations in the literature~\cite{stl-to-milp1,stl-to-milp2}.
They are included in Appendix~\ref{sec:appendix-MILP-encoding} for completeness and correctness.

\noindent\mypara{Graph Operator Encodings}
For agent $i$ at time $t$, the incoming operator 
$\psi = \Inop_{\graphs,[e_1,e_2]}^{W,\#}\varphi$
is evaluated for each graph type $\type\in\settypes$. Its per-type count
contains the incoming neighbors $j\neq i$ such that agent $j$ satisfies
$\varphi$ and the graph-constructor weight
$\gen^{\type}(\jointstate_t,\world_t)(j,i)$ lies in the admissible interval
$W = [w_{\min},w_{\max}]$; the count must lie in $[e_1,e_2]$.
The symbol $\#\in\{\exists,\forall\}$ denotes existential or universal
quantification over $\type\in\settypes$.
The MIP encoding of $\Inop$ proceeds in three steps for a fixed graph type, plus a fourth step that combines per-type encodings under the modal quantifier.
\begin{enumerate}[leftmargin=*, itemsep=0pt, topsep=0pt, parsep=0pt]
\item \textit{Eligible incoming edges.}
For each \(j\in\nodes\setminus\{i\}\), introduce binary variables
\(\lambda_{j,i,t}^{\mathrm{low},\type}\) and
\(\lambda_{j,i,t}^{\mathrm{high},\type}\), indicating satisfaction of the
lower and upper weight bounds, respectively, and an eligibility variable
\(\gamma_{j,i,t}^\type\). We impose
\begin{equation}
\label{eq:enc-in1}
\begin{aligned}
\weights_t^\type(j,i)
&\geq w_{\min}
-M\bigl(1-\lambda_{j,i,t}^{\mathrm{low},\type}\bigr),\\
\weights_t^\type(j,i)
&\leq w_{\min}-\delta_w
+M\lambda_{j,i,t}^{\mathrm{low},\type},\\
\weights_t^\type(j,i)
&\leq w_{\max}
+M\bigl(1-\lambda_{j,i,t}^{\mathrm{high},\type}\bigr),\\
\weights_t^\type(j,i)
&\geq w_{\max}+\delta_w
-M\lambda_{j,i,t}^{\mathrm{high},\type},\\
\gamma_{j,i,t}^\type
\leq a_{j,i,t}^\type,&\quad
\gamma_{j,i,t}^\type
\leq\lambda_{j,i,t}^{\mathrm{low},\type},
\gamma_{j,i,t}^\type
\leq\lambda_{j,i,t}^{\mathrm{high},\type},\\
\gamma_{j,i,t}^\type
&\geq
a_{j,i,t}^\type
+\lambda_{j,i,t}^{\mathrm{low},\type}
+\lambda_{j,i,t}^{\mathrm{high},\type}-2,
\end{aligned}
\end{equation}
where \(\delta_w>0\) is a fixed numerical separation margin.
We assume that feasible edge weights are \(\delta_w\)-separated from the
outside of each interval boundary: a weight below \(w_{\min}\) is at most
\(w_{\min}-\delta_w\), and a weight above \(w_{\max}\) is at least
\(w_{\max}+\delta_w\). Under this convention,
\eqref{eq:enc-in1} enforces
\[
\gamma_{j,i,t}^\type=1
\Leftrightarrow
\left(
a_{j,i,t}^\type=1
\land
w_{\min}\leq\weights_t^\type(j,i)\leq w_{\max}
\right).
\]
\item \textit{Count neighbors satisfying $\varphi$.} Let $z_{\varphi,j,t}\in\{0,1\}$ denote satisfaction of subformula $\varphi$ by
agent $j$ at time $t$.
Introduce $y_{j,i,t}^{\varphi,\type}\in\{0,1\}$ encoding
$\gamma_{j,i,t}^{\type}\wedge z_{\varphi,j,t}$:
\begin{equation}
\label{eq:enc-in2}
\scalebox{0.88}{$\displaystyle
y_{j,i,t}^{\varphi,\type} \le \gamma_{j,i,t}^{\type}, \quad
y_{j,i,t}^{\varphi,\type} \le z_{\varphi,j,t}, \quad
y_{j,i,t}^{\varphi,\type} \ge \gamma_{j,i,t}^{\type}+z_{\varphi,j,t}-1.
$}
\end{equation}
and define the cardinality
\(
c_{i,t}^{\Inop,\varphi,\type}
:= \sum_{j\neq i} y_{j,i,t}^{\varphi,\type}.
\)

\item \textit{Cardinality enforcement.}
First suppose \(e_2<\infty\). Since
\(c_{i,t}^{\Inop,\varphi,\type}\) is integer-valued, introduce binaries
\(\alpha_{i,t}^{\mathrm{low},\type}\) and
\(\alpha_{i,t}^{\mathrm{high},\type}\), indicating
\(c_{i,t}^{\Inop,\varphi,\type}<e_1\) and
\(c_{i,t}^{\Inop,\varphi,\type}>e_2\), respectively:
\begin{equation}
\label{eq:enc-in3}
\begin{aligned}
c_{i,t}^{\Inop,\varphi,\type}
&\leq e_1-1
+M\bigl(1-\alpha_{i,t}^{\mathrm{low},\type}\bigr),\\
c_{i,t}^{\Inop,\varphi,\type}
&\geq e_1-M\alpha_{i,t}^{\mathrm{low},\type},\\
c_{i,t}^{\Inop,\varphi,\type}
&\geq e_2+1
-M\bigl(1-\alpha_{i,t}^{\mathrm{high},\type}\bigr),\\
c_{i,t}^{\Inop,\varphi,\type}
&\leq e_2+M\alpha_{i,t}^{\mathrm{high},\type}.
\end{aligned}
\end{equation}
The per-type satisfaction variable is encoded as
\begin{equation}
\label{enc-in4}
\begin{aligned}
z_{\psi,i,t}^\type
&\leq1-\alpha_{i,t}^{\mathrm{low},\type},\quad
z_{\psi,i,t}^\type
\leq1-\alpha_{i,t}^{\mathrm{high},\type},\\
z_{\psi,i,t}^\type
&\geq
1-\alpha_{i,t}^{\mathrm{low},\type}
-\alpha_{i,t}^{\mathrm{high},\type}.
\end{aligned}
\end{equation}

For a lower-bounded interval \([e_1,\infty)_{\mathbb N}\), only the lower
violation variable is required:
\begin{equation}
\label{eq:enc-in-lower-unbounded}
\begin{aligned}
c_{i,t}^{\Inop,\varphi,\type}
&\leq e_1-1
+M\bigl(1-\alpha_{i,t}^{\mathrm{low},\type}\bigr),\\
c_{i,t}^{\Inop,\varphi,\type}
&\geq e_1-M\alpha_{i,t}^{\mathrm{low},\type}, \quad
z_{\psi,i,t}^\type
=1-\alpha_{i,t}^{\mathrm{low},\type}.
\end{aligned}
\end{equation}


\item \textit{Quantifying Operators} The existential operator
$\psi = \Inop_{\graphs,E}^{W,\exists}\varphi$ requires that at least one
graph type in $\settypes$ satisfies the counting property. For each
$\type\in\settypes$, the per-type encoding above produces a satisfaction
variable $z_{\psi,i,t}^{\type}$; we then enforce the disjunction over types:
\begin{equation}
\label{eq:enc-in5}
\scalebox{0.92}{$\displaystyle
z_{\psi,i,t} \ge z_{\psi,i,t}^{\type} \  \forall \type \in \settypes, \quad
z_{\psi,i,t} \le \sum_{\type \in \settypes} z_{\psi,i,t}^{\type}.
$}
\end{equation}
These constraints ensure $z_{\psi,i,t} = 1$ iff some graph satisfies the property.
\end{enumerate}

The universal operator
$\psi = \Inop_{\graphs,E}^{W,\forall}\varphi$ requires that all graph types
in $\settypes$ satisfy the counting property. This is encoded using a
conjunction instead of the disjunction over graph types.
\begin{equation}
\label{eq:enc-in6}
\scalebox{0.9}{$\displaystyle
z_{\psi,i,t} \le z_{\psi,i,t}^{\type} \;\; \forall \type \in \settypes, \quad
z_{\psi,i,t} \ge 1 - |\settypes| + \sum_{\type \in \settypes} z_{\psi,i,t}^{\type}.
$}
\end{equation}
The outgoing operator is obtained from the incoming case by replacing $(j,i)$ with $(i,j)$ throughout.

\begin{lemma}
The planning problem for agent-local STL-GO specifications, for multi-agent systems whose agent dynamics are described by discrete-time affine difference equations, can be encoded into a MIP such that any satisfying assignment yields a trajectory satisfying the specification.
\label{lemma:milp-graph}
\end{lemma}

\begin{proof}
In the interest of space, we provide a proof sketch, with the full proof in Appendix~\ref{sec:appendix-theory-milp}.
We establish the invariant
\begin{equation}
    z_{\psi, i, t} = 1 \iff (\mas, i, t) \models \psi
  \label{eq:milp-invariant}
\end{equation}
\noindent for every agent-local subformula $\psi$, agent $i$, and time $t$, by structural induction on $\psi$.
\begin{enumerate}[wide, labelwidth=!, labelindent=0pt,nosep]
\item 
The dynamics constraint~\eqref{eq:affine-agent-dyn} is an equality, so feasible assignments correspond to valid trajectories of $F$. 
\item Atomic predicates, Boolean connectives, and the until operator are encoded by standard Big-$M$ and unrolling constructions~\cite{stl-to-milp1,stl-to-milp2}.
\item For
\(\psi=\Inop_{\graphs,[e_1,e_2]}^{W,\#}\varphi\),
MIP-encodability of \(\gen^\type\), for each \(\type\in\settypes\), provides exact edge indicators and
PWA weight expressions. Constraint \eqref{eq:enc-in1} enforces
that \(\gamma_{j,i,t}^\type=1\) iff \((j,i)\in\edge_t^\type\) and its weight
lies in \(W\). Constraint \eqref{eq:enc-in2} then counts exactly those
eligible neighbors that satisfy \(\varphi\).
Equations \eqref{eq:enc-in3}--\eqref{enc-in4} encode a finite cardinality
interval, while \eqref{eq:enc-in-lower-unbounded} encodes
\([e_1,\infty)_{\mathbb N}\).
Finally, \eqref{eq:enc-in5}--\eqref{eq:enc-in6} encode existential or
universal quantification over graph types. The outgoing case is symmetric.
 
\end{enumerate}

Conjoining all constraints, the planning problem for agent-local STL-GO specifications is the feasibility of a MIP whose decision variables include the per-step actions.
\end{proof}
\subsection{Encoding Multi-Agent Quantifiers}\label{sec:quant-specs-milp}
\mypara{Existential Quantification over agents} The existential quantifier $\EX \varphi$ holds when at least one agent $i \in \nodes$ satisfies $\varphi$.
Let $\psi := \EX \varphi$. We encode $\psi$ as:
\begin{equation}
    z_{\psi,t} \ge z_{\varphi,i, t} , \ \  \forall i\in \nodes ;\ \   z_{\psi,t} \le  \sum_{i\in \nodes} z_{\varphi,i, t}
\end{equation}
\mypara{Universal Quantification over agents} The universal quantifier $\FA \varphi$ holds when $\varphi$ holds for all agents $i \in \nodes$.
Let $\psi := \FA \varphi$. We encode $\psi$ as:
\begin{equation}
    z_{\psi,t} \le z_{\varphi,i, t} , \ \  \forall i\in \nodes ; \qquad  z_{\psi,t} \ge 1 - |\nodes| +  \sum_{i\in \nodes} z_{\varphi,i, t}
\end{equation}

\begin{theorem}
    For an STL-GO specification $\phi$ and horizon $T$, if the multi-agent system dynamics and graph constructor are MIP-encodable as described above, and the corresponding MIP encoding is feasible, then the resulting state trajectory $\{\jointstate_t\}_{t=0}^{T}$ satisfies $\phi$, i.e., $(\mas,0)\models \phi$. 
    \label{thm:milp-soundness}
\end{theorem}
\begin{proof}[Proof Sketch]
Lemma~\ref{lemma:milp-graph} establishes~\eqref{eq:milp-invariant} for agent-local subformulas. We extend it to multi-agent subformulas $\psi$ as $z_{\psi, t} = 1 \iff (\mas, t) \models \psi$ by structural induction: joint atomic predicates, embedding $i.\varphi$, Boolean connectives, and until reuse the agent-local arguments on multi-agent variables; $\EX$ and $\FA$ are encoded as disjunction and conjunction over $\{z_{\varphi, i, t}\}_{i \in \nodes}$ (Section~\ref{sec:quant-specs-milp}), matching the semantics by the agent-local invariant. Feasibility with $z_{\phi, 0} = 1$ as a constraint yields $(\mas, 0) \models \phi$. The full proof is in Appendix~\ref{sec:appendix-theory-milp}.
\end{proof}

\section{SMT-based Encoding for STL-GO Planning}
\label{sec:stlgo-smt}

We encode the multi-agent planning problem as a quantifier-free SMT instance in the theory of Linear Real Arithmetic with integers (LRA + LIA), where a satisfying assignment yields a control sequence whose execution satisfies the STL-GO specification.
\subsection{Encoding System Specifications}

\noindent\mypara{State and Control Constraints}
To ensure physical and operational feasibility, agent states are constrained to a hyper-rectangular workspace $\mathcal{X}_{\mathrm{ws}} \subset \mathbb{R}^{n_x}$ with componentwise bounds, and control inputs are bounded componentwise by actuator limits (e.g., maximum velocity, thrust, steering angle).
\begin{equation}
\label{eq:smt-state-input-bounds}
\bigwedge_{t=0}^{T} \bigwedge_{i\in\Vc} \left( \mathbf{x}_{\min} \le x_{t}^i \le \mathbf{x}_{\max} \right), \;
\bigwedge_{t=0}^{T-1} \bigwedge_{i\in\Vc} \left( \mathbf{u}_{\min} \le u_t^i \le \mathbf{u}_{\max} \right)
\end{equation}

\noindent\mypara{System Dynamics}
We assume the system follows deterministic discrete-time affine dynamics:\footnote{The encoding and soundness result presented here use affine dynamics and
LRA+LIA. The same structural translation could instead be instantiated using
an SMT solver over nonlinear arithmetic, such as dReal, to support non-affine
dynamics; we leave that extension for future work.}
\begin{equation}
\label{eq:smt-dynamics}
\agentstate_{t+1}^i = A \agentstate_t^i + B \agentinput_t^i + E \world_t + c \quad 
\forall i \in \Vc, \quad \forall t = 0, 1, \ldots, T-1
\end{equation}


\noindent\mypara{Interaction Graphs}
We focus on interaction graphs whose structure and edge weights depend only on
the joint agent state and world state via the graph constructor function. For
each time $t$ and ordered pair $(i,j)$, let
$\eta_{i,j}^{\type}(\jointstate_t,\world_t)$ be the edge-existence predicate
and let $e_{i,j}^{\type}(\jointstate_t,\world_t)$ compute the corresponding
edge weight. We say $\gen^{\type}$ is \emph{SMT-encodable} if, for every
$(i,j)$, $\eta_{i,j}^{\type}$ is a Boolean formula over LRA+LIA and
$e_{i,j}^{\type}$ is an LRA term. Both can then be encoded directly in the
quantifier-free SMT instance.

\subsection{Encoding Agent-local Operators}
\label{sec:smt-agent-local}

For each subformula $\psi$ of specification $\varphi$, agent $i \in \Vc$, and time $t \in T$,
we introduce a Boolean variable $z_{\psi,t}^i \in \{\top, \bot\}$ with the intended semantics
\begin{equation}
z_{\psi,t}^{i} = \top
\;\Longleftrightarrow\;
(\mas,i,t) \models \psi
\end{equation}
The encoding of atomic predicates, logical operators, and temporal operators
closely follow established formulations in the literature~\cite{momtaz2023monitoring,prabhakar2018automatic}.
They are included in Appendix~\ref{sec:appendix-SMT-encoding} for completeness and correctness.

\noindent\mypara{Graph Operators}

\label{sec:smt-in-out}
For $\psi=\Inop_{\graphs,[e_1,e_2]}^{W,\#}\varphi$, the SMT encoding
proceeds in three steps for each fixed graph type $\type\in\settypes$, plus a
fourth step that combines the per-type encodings under the modal quantifier
$\#$.

\begin{enumerate}[leftmargin=*, itemsep=2pt, topsep=2pt, parsep=0pt]
\item \textit{Eligible incoming edges.}
For each graph type \(\type\in\settypes\), agents \(i,j\in\nodes\) with
\(i\neq j\), and time \(t\), introduce a Boolean variable
\(b_{j,i,t}^{\type}\), indicating that \((j,i)\) is an edge in the graph
instance \(\graph_t^\type\) and that its weight lies in \(W\).
\begin{align}
b_{j,i,t}^{\type}
\;\leftrightarrow\;&
\Bigl(
\eta_{j,i}^{\type}(\jointstate_t,\world_t)
\notag\\
&\quad\wedge\;
w_{\min}
\leq
e_{j,i}^{\type}(\jointstate_t,\world_t)
\leq
w_{\max}
\Bigr).
\label{eq:smt-eligible}
\end{align}

where \(\eta_{j,i}^{\type}\) is the edge-existence predicate and
\(e_{j,i}^{\type}(\jointstate_t,\world_t)\) is the
edge-weight supplied by the graph constructor. If weights are not
of interest and \(W=(-\infty,\infty)\), the weight comparisons are omitted.

\item \textit{Count neighbors satisfying $\varphi$.}
Let $z^j_{\varphi,t}$ be a boolean variable to denote satisfaction of subformula $\varphi$ by
agent $j$ at time $t$.
This is encoded as
\begin{equation}
\label{eq:smt-conjunct}
n^{\type}_{j,i,t} = b^{\type}_{j,i,t} \wedge z^j_{\varphi,t}
\end{equation}

\item \textit{Cardinality.}
For each graph type $\type \in \settypes$, define an integer variable $c^{\type}_{i,t}$ counting eligible $\varphi$-satisfying incoming neighbors:
\begin{equation}
\label{eq:smt-cardinality}
\scalebox{0.85}{$\displaystyle
c^{\type}_{i,t} = \sum_{j \ne i} \mathsf{ite}(n^{\type}_{j,i,t}, 1, 0),
$}
\end{equation}
The per-type satisfaction variable is constrained by
\begin{equation}
\label{eq:smt-cardinality-satisfaction}
z_{\psi,t}^{i,\type}
\;\leftrightarrow\;
\begin{cases}
e_1\leq c_{i,t}^{\type}\leq e_2,
& e_2<\infty,\\
e_1\leq c_{i,t}^{\type},
& e_2=\infty.
\end{cases}
\end{equation}

\item \textit{Quantification over Graph Types.}
For $\# = \exists$ the satisfaction of $\psi$ at agent $i$ is the
disjunction over types of the per-type satisfactions; 
\begin{equation}
\label{eq:smt-graph-quant}
z^i_{\psi, t} \; = \;
\bigvee_{\type \in \settypes} z^{i, \type}_{\psi, t}
\quad (\# = \exists), 
\end{equation}
\end{enumerate}
For $\# = \forall$
it is the conjunction instead of the disjunction.
The encoding of the $\Outop_{\graphs, [e_1, e_2]}^{W, \#} \varphi$ operator is obtained from the incoming case by substituting $(j, i)$ with $(i, j)$ in steps 1--3.
\begin{lemma}
\label{lemma:smt-graph}
The planning problem for agent-local STL-GO specifications, for multi-agent systems whose agent dynamics are described by discrete-time affine difference equations and whose graph constructors are SMT-encodable, can be encoded into a quantifier-free SMT instance such that any satisfying assignment yields a trajectory satisfying the specification.
\end{lemma}
\begin{proof}[Proof sketch]
We establish the invariant
\begin{equation}
    z^i_{\psi, t} = \top \iff (\mas, i, t) \models \psi
  \label{eq:smt-invariant}
\end{equation}
for every agent-local subformula $\psi$, agent $i$, and time $t$, by structural induction on $\psi$.
The dynamics constraint~\eqref{eq:smt-dynamics} is an equality in the SMT variables, so any satisfying assignment corresponds to a valid trajectory of $F$. Atomic predicates, Boolean connectives, and the until operator are encoded by direct LRA constraints and finite unrolling over the bounded horizon~\cite{momtaz2023monitoring,prabhakar2018automatic}.
For $\psi = \Inop_{\graphs, [e_1, e_2]}^{W, \#} \varphi$, SMT-encodability of $\gen^{\type}$, for each $\type\in\settypes$, ensures the edge-existence predicate and edge weight are LRA+LIA expressions; \eqref{eq:smt-eligible}--\eqref{eq:smt-cardinality-satisfaction} enforce $z^{i, \type}_{\psi, t} = \top$ iff the count of eligible $\varphi$-satisfying neighbors lies in $[e_1, e_2]$, and~\eqref{eq:smt-graph-quant} lifts this to $\#$-quantification over graph types. The outgoing case is symmetric. The full proof is in Appendix~\ref{sec:appendix-theory-smt}.
\end{proof}
\subsection{Encoding Multi-Agent Quantifiers}
\label{sec:smt-multi-agent}
The existential quantifier $\EX \varphi$ indicates the existence of at least
one agent $i\in\nodes$ for which $\varphi$ holds. For a formula
$\psi:=\EX\varphi$, we encode $\psi$ as
$z_{\psi,t}=\bigvee_{i\in\nodes}z^i_{\varphi,t}$. The universal quantifier
$\FA\varphi$ is similarly defined using $\bigwedge_{i\in\nodes}$ instead of
$\bigvee_{i\in\nodes}$.


\begin{theorem}
For an STL-GO specification $\phi$ and horizon $T$, if the multi-agent system dynamics and graph constructor are SMT-encodable as described above, and if the SMT instance is satisfiable, then any satisfying assignment yields a trajectory $\{\jointstate_t\}_{t=0}^{T}$ such that $(\mas, 0) \models \phi$.
\label{thm:smt-soundness}
\end{theorem}
\begin{proof} [Proof Sketch]
Lemma~\ref{lemma:smt-graph} establishes~\eqref{eq:smt-invariant} for agent-local subformulas. We extend it to multi-agent subformulas $\psi$ as $z_{\psi, t} = \top \iff (\mas, t) \models \psi$ by structural induction: joint atomic predicates, embedding $i.\varphi$, Boolean connectives, and until reuse the agent-local arguments on multi-agent variables; $\EX$ and $\FA$ are encoded as described above, matching the semantics by the agent-local invariant. Satisfiability of the SMT instance with $z_{\phi, 0} = \top$ as a constraint yields $(\mas, 0) \models \phi$. The full proof is in Appendix~\ref{sec:appendix-theory-smt}.
\end{proof}

\section{Experiments and Results}
\label{sec:exp}
\begin{figure*}[t]
    \centering
    \begin{subfigure}[b]{0.6\columnwidth}
        \includegraphics[width=\textwidth]{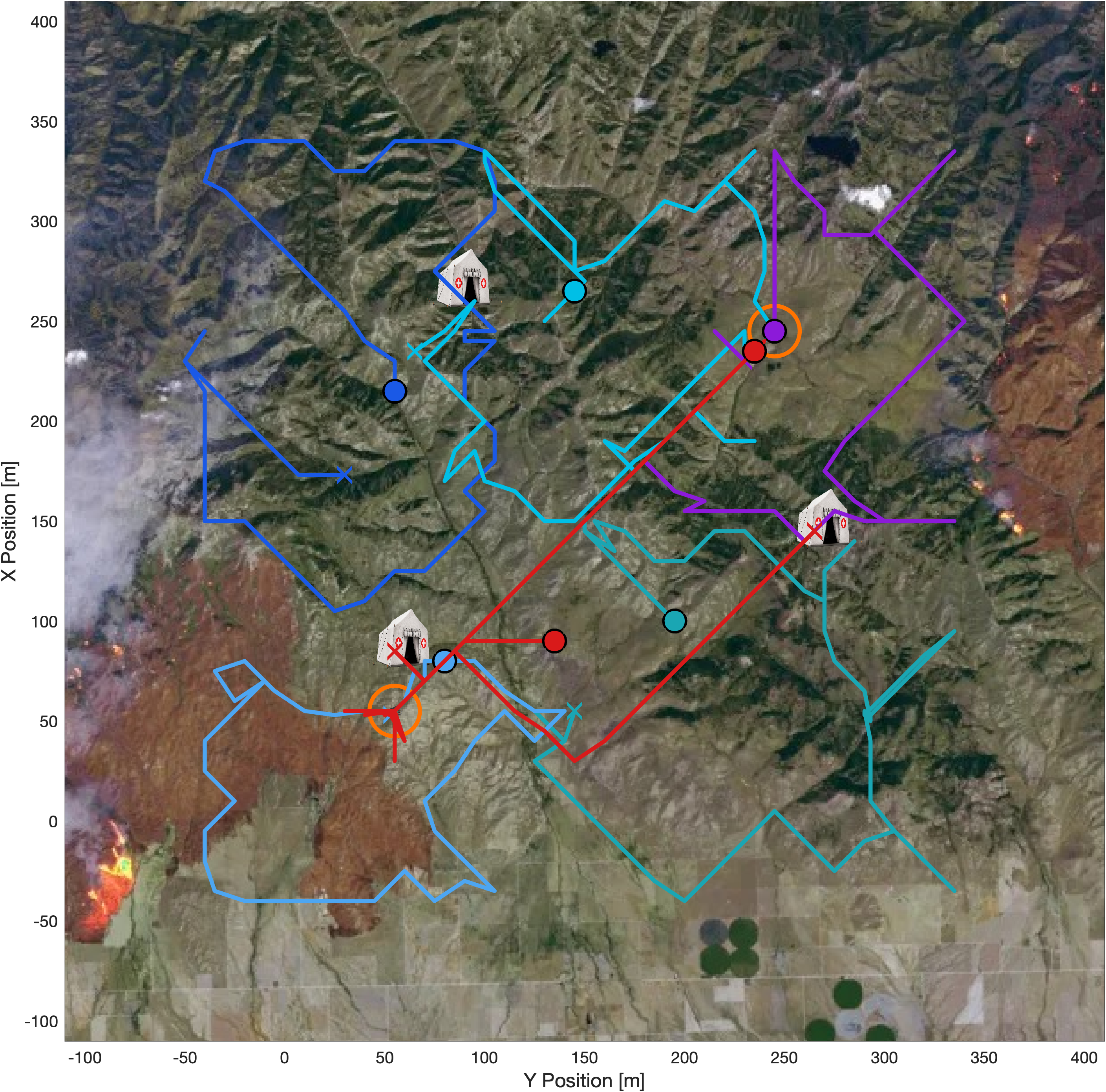}
        \caption{\small No objective}
    \end{subfigure}
    \hfill
    \begin{subfigure}[b]{0.6\columnwidth}
        \includegraphics[width=\textwidth]{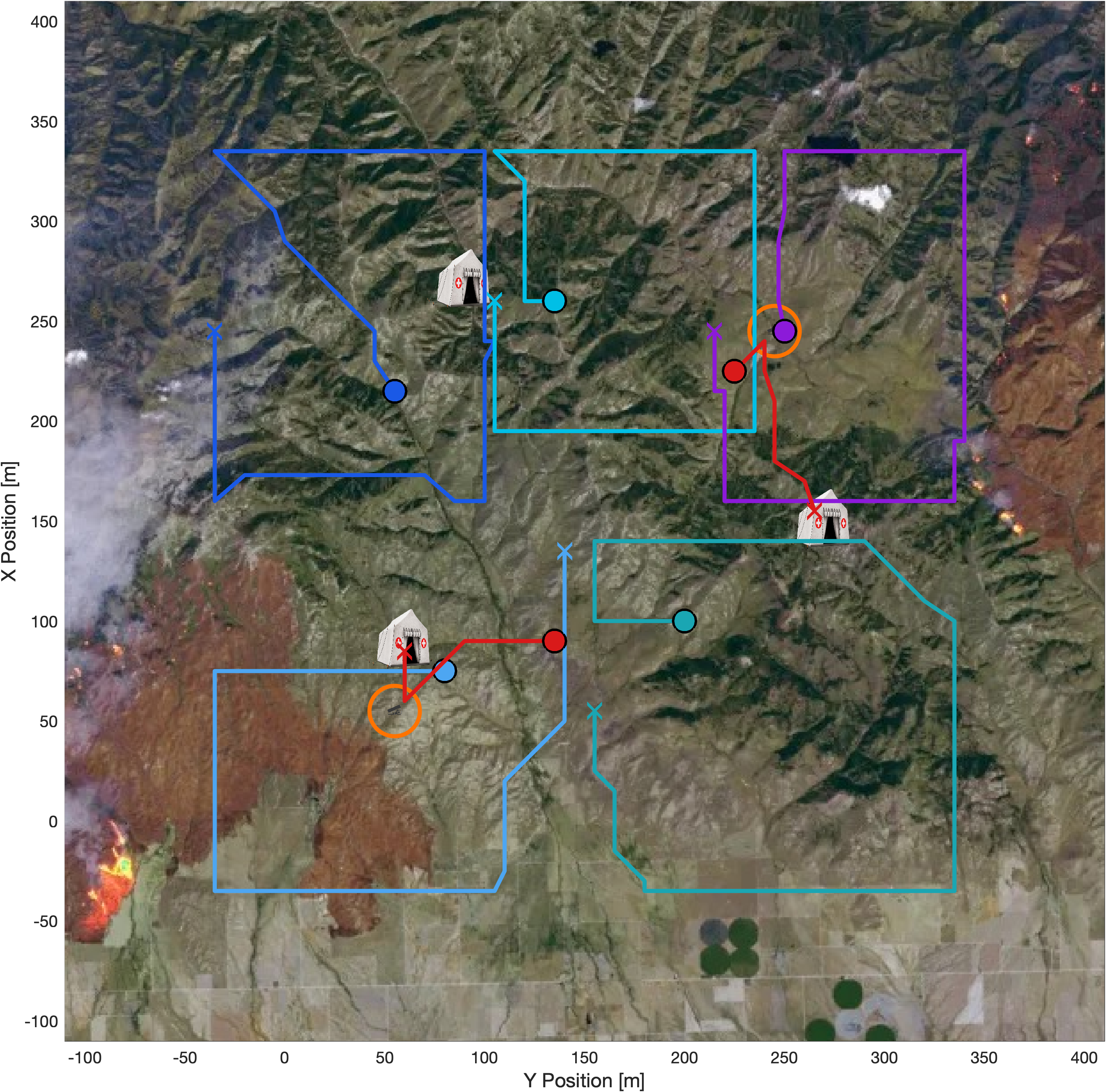}
        \caption{\small Linear objective}
    \end{subfigure}
    \hfill
    \begin{subfigure}[b]{0.6\columnwidth}
        \includegraphics[width=\textwidth]{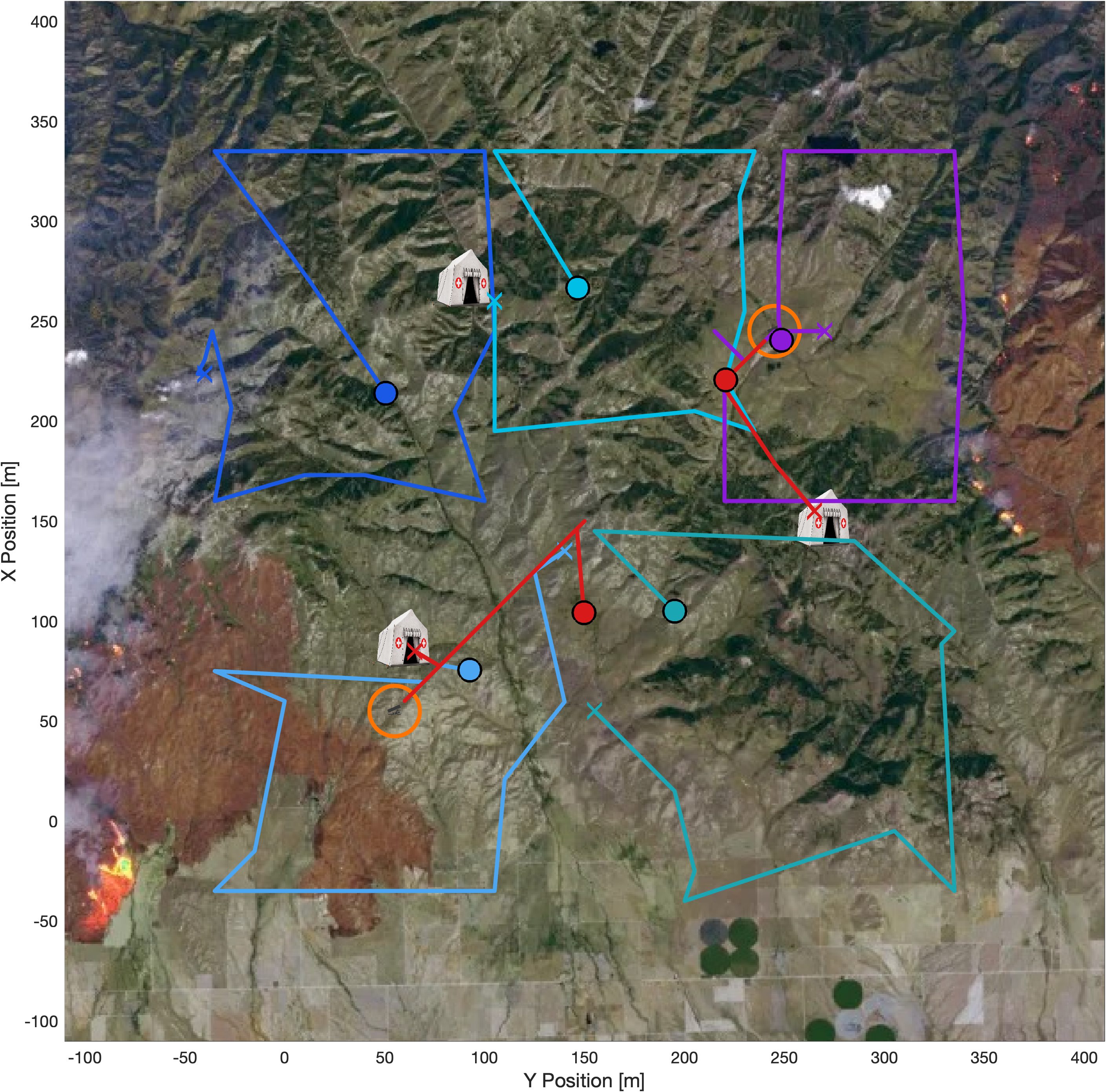}
        \caption{\small Quadratic objective}
    \end{subfigure}
    \caption{\small Effect of objective function on agent trajectories
        (5 locators, 2 rescuers). With no objective~(a), the MIP returns
        an arbitrary feasible solution. A linear objective~(b) and quadratic
        objective~(c) progressively guide the rescuer toward more direct
        paths to the emergency sites.}
    \label{fig:objectives}
\end{figure*}
\begin{figure*}[t]
    \centering
    \begin{subfigure}[b]{0.6\columnwidth}
        \includegraphics[width=\textwidth]{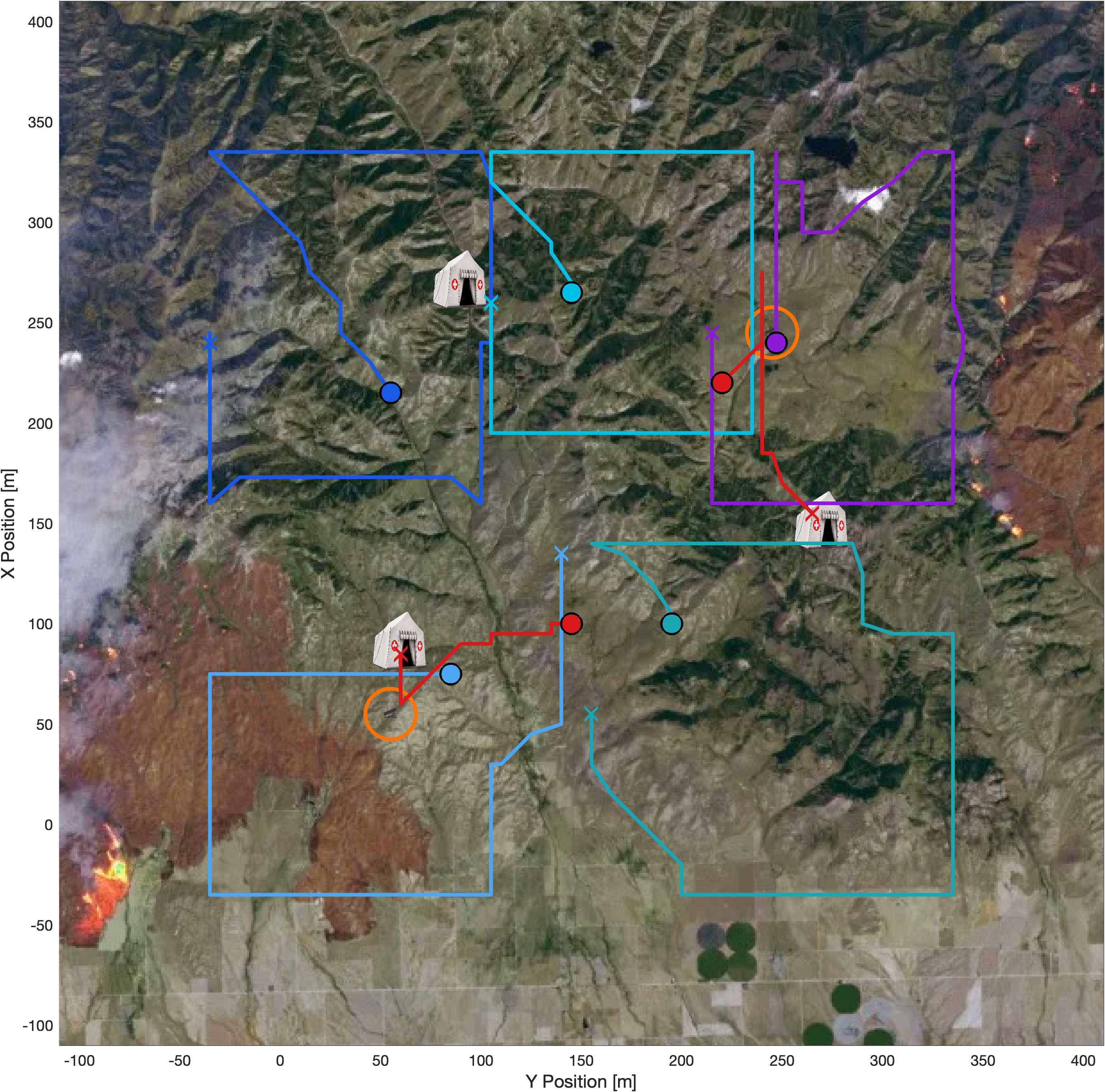}
        \caption{\small 5 locators, 2 rescuers}
    \end{subfigure}
    \hfill
    \begin{subfigure}[b]{0.6\columnwidth}
        \includegraphics[width=\textwidth]{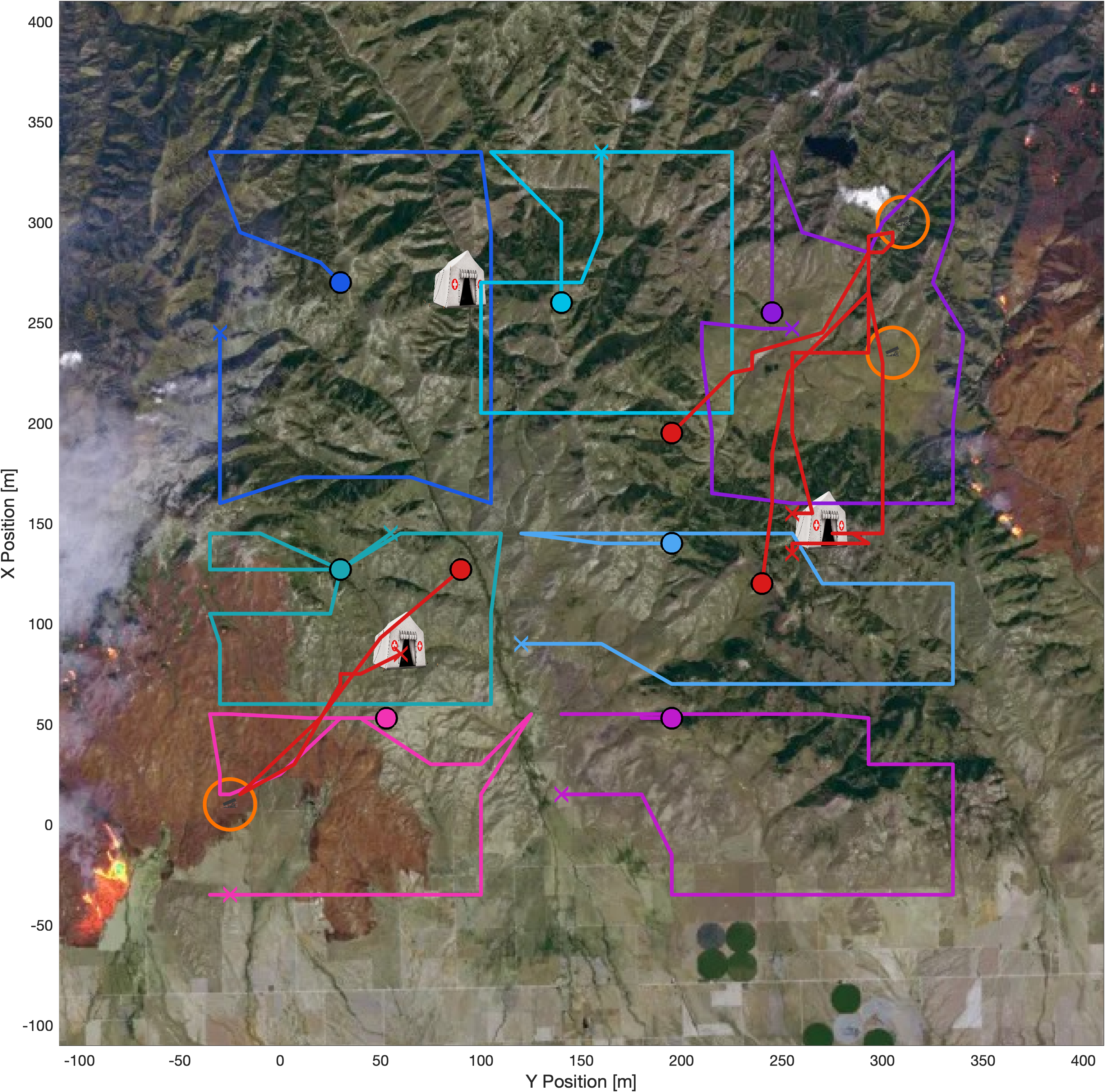}
        \caption{\small 7 locators, 3 rescuers}
    \end{subfigure}
    \hfill
    \begin{subfigure}[b]{0.6\columnwidth}
        \includegraphics[width=\textwidth]{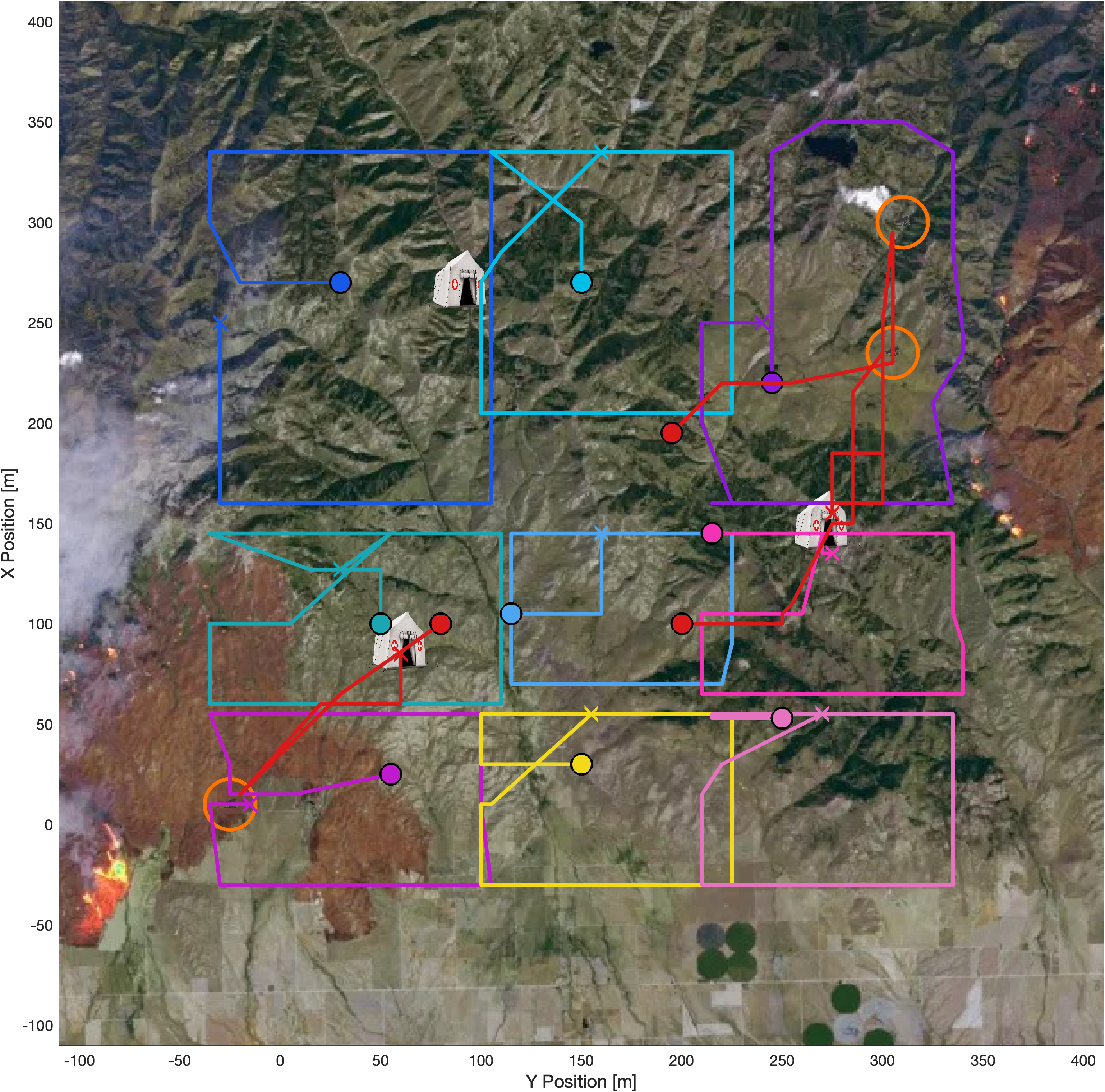}
        \caption{\small 9 locators, 3 rescuers}
    \end{subfigure}
    \caption{\small Scalability across team sizes under a linear objective.
        Increasing the number of locators improves coverage of the monitored
        region, while the rescuers adapt their paths to the denser set of
        detected emergencies.}
    \label{fig:scalability}
\end{figure*}

\begin{table*}[t]
\centering
\caption{Ablation and scalability results (solve time, number of variables and number of constraints) on SwarmLab search-and-rescue contingency-planning simulations across MIP and SMT encodings. ($|\mathcal{L}|$ and $|\mathcal{R}|$ indicate the number of locator and rescuer agents respectively. Reported times are totals over the scenario sweep ($|\Omega|{=}4$ for $|\mathcal{L}|{=}5$, $|\Omega|{=}8$ otherwise); reported model sizes are those of the largest per-scenario program. $^\dagger$Per-scenario time limit reached; best incumbent reported. SMT is satisfaction-only; objective ablations apply only to MIP.)}
\label{tab:ablation_scaling}
\setlength{\tabcolsep}{2.5pt}
\begin{tabular}{@{}llcccccccccccccc@{}}
\toprule
& & & \multicolumn{3}{c}{\textbf{MIP (No Objective)}}
& \multicolumn{3}{c}{\textbf{MIP (Linear Objective)}}
& \multicolumn{3}{c}{\textbf{MIP (Quad. Objective)}}
& \multicolumn{3}{c}{\textbf{SMT}} \\
\cmidrule(lr){4-6} \cmidrule(lr){7-9} \cmidrule(lr){10-12} \cmidrule(lr){13-15}
\textbf{Condition}
& $\boldsymbol{|\mathcal{L}|}$
& $\boldsymbol{|\mathcal{R}|}$
& \textbf{Time (s)} & \textbf{|Vars|} & \textbf{|Constr|}
& \textbf{Time (s)} & \textbf{|Vars|} & \textbf{|Constr|}
& \textbf{Time (s)} & \textbf{|Vars|} & \textbf{|Constr|}
& \textbf{Time (s)} & \textbf{|Vars|} & \textbf{|Constr|} \\
\midrule
\multirow{3}{*}{\shortstack[l]{{STL only} \\ \textit{(No graph operators)}}}
& 5 & 2 &  69 &  60k & 125k & 1862$^\dagger$ &  60k & 126k & 2591$^\dagger$ &  60k & 126k &  0.66 & 21.3k & 23.2k \\
& 7 & 3 & 177 & 105k & 146k & 3759$^\dagger$ & 106k & 148k & 5216$^\dagger$ & 106k & 148k &  1.49 & 23.8k & 26.6k \\
& 9 & 3 & 277 & 146k & 210k & 3877$^\dagger$ & 147k & 212k & 5300$^\dagger$ & 147k & 212k &  1.74 & 35.2k & 38.5k \\
\midrule
\multirow{3}{*}{\shortstack[l]{\textbf{$\mathcal{G}^{s}$} \\ \textit{(Sensing)}}}
& 5 & 2 & 151 &  74k & 143k & 1919$^\dagger$ &  74k & 144k & 2659$^\dagger$ &  74k & 144k &  2.13 & 23.2k & 25.2k \\
& 7 & 3 & 240 & 135k & 185k & 3846$^\dagger$ & 136k & 187k & 5311$^\dagger$ & 136k & 187k &  3.15 & 27.9k & 30.7k \\
& 9 & 3 & 401 & 182k & 257k & 4016$^\dagger$ & 183k & 259k & 5434$^\dagger$ & 183k & 259k &  2.61 & 40.1k & 43.4k \\
\midrule
\multirow{3}{*}{\shortstack[l]{\textbf{$\mathcal{G}^{s} + \mathcal{G}^{c}$} \\ \textit{(Sens.\ + Comm.)}}}
& 5 & 2 & 175 &  81k & 153k & 2032$^\dagger$ &  81k & 154k & 2635$^\dagger$ &  81k & 154k &  1.29 & 25.2k & 27.1k \\
& 7 & 3 & 362 & 155k & 209k & 3940$^\dagger$ & 155k & 211k & 5393$^\dagger$ & 155k & 211k &  1.80 & 32.1k & 34.8k \\
& 9 & 3 & 459 & 204k & 284k & 4154$^\dagger$ & 206k & 287k & 5786$^\dagger$ & 206k & 287k &  4.06 & 45.1k & 48.4k \\

\midrule
\multirow{3}{*}{\shortstack[l]{\textbf{$\mathcal{G}^{s}, \mathcal{G}^{c}, \mathcal{G}^{\mathrm{task}}$} \\ \textit{(Sens. + Comm + Task)}}}
& 5 & 2 &  340 &  86k & 161k & 1201$^\dagger$ &  87k & 162k & 2234$^\dagger$ &  86k & 161k &  3.56 & 22.0k & 24.0k \\
& 7 & 3 &  775 & 172k & 230k & 3750$^\dagger$ & 173k & 232k & 5809$^\dagger$ & 173k & 232k &  5.68 & 46.9k & 49.8k \\
& 9 & 3 & 1480 & 225k & 307k & 4973$^\dagger$ & 226k & 309k & 5997$^\dagger$ & 226k & 309k & 16.50 & 58.2k & 61.6k \\

\bottomrule
\end{tabular}
\end{table*}

\noindent We implement a planner for the aforementioned search-and-rescue
example. Let \(\mathcal M\) denote a fixed set of known assembly areas where
emergencies may occur. At planning time, the location of every assembly area is
known, but it is unknown whether each assembly area contains an active
emergency. Locator agents patrol the assembly areas and observe their activation
statuses when they enter the sensing range. An active emergency must then be
communicated, assigned to one or more rescuers, and resolved within bounded time
by transporting the rescued individual to the rescue center \(\mathcal C\). To
account for potential emergencies, we construct plans for all possible
emergency-activation scenarios.

The experimental instances use role-specific admissible control sets while
retaining the same motion-model structure for all agents.\footnote{Locator controls satisfy \(\agentinput_t^\ell\in\mathcal U_{\mathcal L}\), whereas rescuer controls satisfy \(\agentinput_t^r\in\mathcal U_{\mathcal R}\), with the maximum rescuer speed strictly lower than the maximum locator speed. Thus, the experimental role distinction includes different mobility limits in addition to different mission obligations.}

Let \(\Omega\subseteq\{0,1\}^{|\mathcal M|}\) denote the set of scenarios, where
\(\omega_m=1\) indicates that assembly area \(m\) contains an active emergency
in scenario \(\omega\), and \(\omega_m=0\) indicates that it does not. If all
activation combinations are considered, then \(|\Omega|=2^{|\mathcal M|}\).

For each scenario \(\omega\in\Omega\), the encoding contains a corresponding
state and control sequence. These scenario-indexed sequences are solved
jointly. Scenarios with identical locator
observation histories must have identical control inputs, i.e., if
scenarios \(\omega\) and \(\omega'\) are indistinguishable from the agents at
time \(t\), then
\(\jointcontrol_t^\omega=\jointcontrol_t^{\omega'}\). The plans may branch
only after a locator observation distinguishes the scenarios.

The resulting solution is a contingent plan represented as a finite scenario
tree. At runtime, the agents initially execute the common plan prefix. When a
locator observes whether an assembly area is active, the branch consistent
with that observation is selected, and execution continues along that branch.
Further observations may select subsequent branches.

The predicates and specifications below are instantiated for every
\(\omega\in\Omega\). We suppress the scenario superscript when it is clear
from context. In scenario \(\omega\),
\(\varphi_m^{\mathrm{emg}}\) is true iff \(\omega_m=1\).

We introduce the following atomic predicates:
$\varphi^{\mathrm{emg}}_m$ (emergency active),
$\varphi^{\mathrm{near}}_{r,m}$ (rescuer near emergency),
$\varphi^{\mathrm{atC}}_r$ (rescuer at rescue center), and
$\varphi^{\mathrm{carry}}_r$ (rescuer carrying an individual).
The joint geometric predicate
$\varphi^{\mathrm{sense}}_m$ states that at least one locator lies within
sensing range of emergency $m$. Graph-derived predicates capture the remaining
collective and relational conditions:
$\varphi^{\mathrm{LL}}_\ell$ (locator-to-locator communication) and
$\varphi^{\mathrm{LR}}_\ell$ (locator-to-rescuer communication).
Let \(\varphi_{\ell,r,m}^{\mathrm{task}}\) denote the predicate that locator
\(\ell\) has assigned rescuer \(r\) specifically to emergency \(m\) at the
current time. Its truth implies that the corresponding task edge
\((\ell,r)\in\edge_t^{\mathrm{task}}\) is active.

We fix horizons
$T_{\mathrm{det}},T_{\mathrm{assign}},T_{\mathrm{relay}},
T_{\mathrm{reach}},T_{\mathrm{deliver}}\in\mathbb{N}$. The specifications for the mission are then outlined as follows.
\begin{enumerate}[leftmargin=*, labelindent=0pt]
  \item \textit{Bounded emergency detection.}
        Every emergency must be detected by the locator swarm within bounded time:
\[
          \phi_{\mathrm{detect}}
          := \Globally\;
          \bigwedge_{m\in\mathcal{M}}
          \Big(
          \varphi^{\mathrm{emg}}_m
          \Rightarrow
          \Eventually_{[0,T_{\mathrm{det}}]}\;
          \varphi^{\mathrm{sense}}_m
          \Big).
\]     
\item \textit{Detection-to-relay.}
Once an emergency is detected, the information must be communicated to either
another locator or a rescuer within bounded time:
\[
\resizebox{\linewidth}{!}{$\displaystyle
  \phi_{\mathrm{relay}}
  := \Globally\;
  \bigwedge_{m\in\mathcal{M}}
  \Big(
  \varphi^{\mathrm{emg}}_m \wedge \varphi^{\mathrm{sense}}_m
  \Rightarrow
  \Eventually_{[0,T_{\mathrm{relay}}]}\,
  \bigvee_{\ell\in\mathcal L}
  (\varphi^{\mathrm{LL}}_\ell \,\vee\, \varphi^{\mathrm{LR}}_\ell)
  \Big).
$}
\]
  \item \textit{Detection-to-assignment.}
        Once an emergency is detected, at least one locator must assign a rescuer within
        bounded time:
\[
\resizebox{\linewidth}{!}{$\displaystyle
          \phi_{\mathrm{assign}}
          := \Globally\;
          \bigwedge_{m\in\mathcal{M}}
          \Big(
          \varphi^{\mathrm{emg}}_m \wedge \varphi^{\mathrm{sense}}_m
          \Rightarrow
          \Eventually_{[0,T_{\mathrm{assign}}]}\;
          \bigvee_{\ell\in\mathcal L}\;
          \bigvee_{r\in\mathcal R}
          \varphi^{\mathrm{task}}_{\ell,r,m}
          \Big).
$}
\]
  \item \textit{Rescuer response time.}
        If a locator assigns a rescuer, that rescuer must reach the emergency within a
        bounded time:
     \begin{align*}
         \phi_{\mathrm{reach}}
:=
\Globally
\bigwedge_{\ell\in\mathcal L}
\bigwedge_{r\in\mathcal R}
\bigwedge_{m\in\mathcal M}
\Big(
\varphi_{\ell,r,m}^{\mathrm{task}}
\wedge
\varphi_m^{\mathrm{emg}}
\Rightarrow \\
\Eventually_{[0,T_{\mathrm{reach}}]}
\varphi^{\mathrm{near}}_{r,m}
\Big).
     \end{align*}

  \item \textit{Rescue and deliver.}
        After reaching the emergency, the rescuer must deliver the rescued individual to the rescue
        center:
\[
 \begin{aligned}
 \phi_{\mathrm{deliver}}
      & := \Globally\;
      \bigwedge_{\ell\in\mathcal L}
      \bigwedge_{r\in\mathcal R}
      \bigwedge_{m\in\mathcal M}
      \Big(
      \varphi_{\ell,r,m}^{\mathrm{task}}
      \wedge
      \varphi^{\mathrm{near}}_{r,m}
      \Rightarrow \\
      & \quad\quad \Eventually_{[0,T_{\mathrm{deliver}}]}\;
      (\varphi^{\mathrm{carry}}_r \wedge \varphi^{\mathrm{atC}}_r)
      \Big).
\end{aligned}
 \]
\end{enumerate}
\noindent For each scenario \(\omega\in\Omega\), let
\(\phi^\omega\) denote the conjunction of the five mission specifications
above under the activation assignment \(\omega\). The scenario-based planning
problem requires
\(\bigwedge_{\omega\in\Omega}\phi^\omega\).
\noindent All experiments were conducted using Gurobi~\cite{gurobi} (for MIP) and Z3~\cite{z3} (for SMT), and were executed on a compute cluster with 16 CPU cores and 64 GB of memory. The simulations were performed using a customized SwarmLab~\cite{swarmlab} framework, extended to incorporate our environment and dynamics and to execute the synthesized trajectories.\footnote{Satellite terrain image courtesy of NASA Earth Observatory (\url{https://earthobservatory.nasa.gov}).}

\noindent\textbf{Effect of objective function.}
Figure~\ref{fig:objectives} shows the effect of the objective on rescuer trajectories with 5 locators and 2 rescuers. Without an objective (a), the MIP returns an arbitrary feasible plan; the linear objective (b) and quadratic objective (c) progressively shape the rescuer toward more direct paths.\footnote{The linear objective is the rescuer $L_1$ path cost $J_1=\sum_{r\in\mathcal R}\sum_{t=0}^{T-1}\|\pos_{t+1}^r-\pos_t^r\|_1$. The quadratic objective is the squared $L_2$ step cost $J_2=\sum_{r\in\mathcal R}\sum_{t=0}^{T-1}\|\pos_{t+1}^r-\pos_t^r\|_2^2$.} Since SMT is satisfaction-only, this comparison is specific to the MIP encoding and motivates retaining MIP for objective-driven planning despite its larger encoding footprint.

\noindent \textbf{Scalability Evaluation with Increasing Graph Complexity and Team Size.}
We evaluate the incremental 
impact of graph-dependent constraints on planning complexity by progressively increasing specification complexity while keeping the environment, dynamics, and objective fixed. Starting from a baseline of STL predicates (no graph operators), we add (i) sensing-neighborhood constraints from a time-varying sensing graph, (ii) communication-neighborhood constraints from a time-varying communication graph, and (iii) task constraints from a time-varying task graph. For each case, we further vary the number of locator and rescuer agents. (See Fig.~\ref{fig:scalability})
 
\noindent \textbf{Results.}
Results are presented in Table~\ref{tab:ablation_scaling}, which compares the numbers of variables and constraints and the time taken to find a solution across all ablations for both methods. STL specifications without graph operators scale comparatively well across both solvers and serve as a clear baseline. Adding sensing graphs increases complexity by coupling continuous agent states with logical satisfaction; communication graphs amplify this effect by enforcing pairwise proximity constraints that require joint reasoning over agent configurations; task graphs with decision-dependent assignments yield a large increase in solve time and problem size for both solvers.
Quantitatively, SMT encodings result in fewer variables and constraints and exhibit faster solve times. MIP encodings, despite larger size and longer solve times, are required for objective-driven planning (Fig.~\ref{fig:objectives}); for the hardest configuration ($|\mathcal{L}|=9, |\mathcal{R}|=3$, all graphs, quadratic objective), the time limit was reached before optimality could be established, and Table~\ref{tab:ablation_scaling} reports the best incumbent.




\section{Discussion}
\label{sec:conc}
\noindent \textbf{Related Work.} 
STL has been widely used for multi-agent control and planning via optimization-based synthesis, including robustness-aware feedback formulations~\cite{FormalMethodsMultiAgent,PPCSTL} and MIP encodings with abstractions such as timed waypoints~\cite{MultiAgentSTLWaypoints}.
These approaches focus on individual agent trajectories and continuous dynamics, and do not explicitly capture graph-based spatial relations or collective constraints. To incorporate spatial structure, spatio-temporal logics such as STREL introduce graph-based reachability and escape operators~\cite{STREL,STRELDynamicNetworks}, enabling the specification and monitoring of collective behaviors.
Learning-based and synthesis approaches from STREL have also been explored~\cite{NNSTREL}, but these works do not address optimization-based planning via MIP or SMT.

\noindent Several works employ SMT- or SAT-based encodings for multi-agent planning under temporal logic constraints. Compositional synthesis from safe LTL fragments using SMT improves scalability and correctness guarantees~\cite{SMTMultiRobotSafeLTL}, while online or incremental planning under LTL has also been studied~\cite{OnlineMultiRobotLTL}. Optimization-based planning has further been explored via lazy SMT and MIP formulations~\cite{shoukry2016scalable,SMTMultiRobotSafeLTL} and SAT-based convex optimization~\cite{shoukry2017linear}, while MIP-based waypoint formulations enable long-horizon multi-agent STL planning~\cite{MultiAgentSTLWaypoints}.
Capability Temporal Logic (CaTL) integrates temporal logic with task allocation, routing, and resource constraints for heterogeneous teams~\cite{ProbabilisticCaTLCoordination,CaTLResourceConstraints}, typically formulating planning as a combinatorial optimization problem over assignments and schedules.
While CaTL adopts a centralized planning perspective similar to ours, its operators focus on capabilities and task satisfaction rather than spatio-temporal reasoning over interaction graphs.

\noindent \textbf{Comparison with HyperLTL.} Hyperproperties have also been used to describe specifications for
multi-agent systems~\cite{hsu2025hyprl,wang2020hyperproperties,finkbeiner2023logics}.
Closest to our setting is the work
of~\cite{wang2020hyperproperties}, who use HyperLTL to plan for
multi-robot systems with relational objectives, and
HypRL~\cite{hsu2025hyprl}, which learns control policies from
hyperproperty specifications in the decentralized model-free setting (ours is centralized and model-based).

\noindent Beyond this, HyperLTL is a propositional logic: it
admits only Boolean atomic propositions, and joint predicates over
multiple agents reduce to syntactic sugar for disjunctions of single-agent
predicates over discrete, grid-like environments. 
\noindent Encoding our specifications in HyperLTL also requires \emph{reifying} every pointwise agent choice appearing inside a temporal operator, since HyperLTL admits only trace quantifiers at the outermost prefix. For example, the locator and rescuer choices in the assignment specification $\phi_{\mathrm{assign}}$ are expanded as disjunctions over the fixed sets $\mathcal L$ and $\mathcal R$,
\[
  \phi_{\mathrm{assign}}
  =
  \begin{multlined}[t]
  \forall \pi_m.\,
  \Globally\Big(
    \varphi^{\mathrm{emg},\pi_m}\wedge
    \varphi^{\mathrm{sense},\pi_m}
    \!\Rightarrow\!\\
    \Eventually_{[0,T_{\mathrm{assign}}]}\!
    \bigvee_{\ell \in \mathcal{L}}\!
    \bigvee_{r \in \mathcal{R}}\!
    \mathsf{task}^{\pi_\ell,\pi_r,\pi_m}
  \Big),
  \end{multlined}
\]
or a Skolemization that lifts the existentials to the prefix at the cost
of one alternation, yielding a $\forall\exists$ formula outside the
alternation-free fragment. The reach and delivery formulas, by contrast,
use fixed universal prefixes but require joint predicates such as
$\varphi^{\mathrm{task}}_{\ell,r,m}$ and
$\varphi^{\mathrm{near}}_{r,m}$ to be represented as multi-trace atomic
propositions. STL-GO evaluates these finite agent choices pointwise and
supports joint predicates directly; the full reifications are deferred to
Appendix~\ref{app:hyperprops}.

\noindent Additionally, to compare the encodings of HyperLTL and STL-GO empirically, we
adapt the wildfire-rescue grid-world benchmark from
HypRL~\cite{hsu2025hyprl}, where the authors consider
$\mathcal{N}\times \mathcal{N}$ grids where $3\leq\mathcal{N}\leq 10$. A firefighter agent must extinguish all fire
cells while a medic agent rescues all victim cells; the two agents must
stay within a bounded communication range, and the medic cannot
enter a fire cell until the firefighter has extinguished it. We encode the
mission in HyperLTL and STL-GO and solve each under MIP and SMT
backends; the full specification and encoding details
are reported in Appendix~\ref{app:hyperprops}.
The results in Table~\ref{tab:hyperltl-scaling}
show that as the grid grows, HyperLTL encodings incur an $O(T \cdot N^4)$
constraint blow-up from the pairwise-cell enumeration required by the
communication-range constraint: by $N{=}10$, HyperLTL + SMT
generates 355k constraints versus 2.4k for STL-GO + SMT. STL-GO's graph operators encode the same coordination through
real-valued predicates, with constraint counts growing
linearly in $T$.
\begin{table}[t]
\centering
\caption{Comparison of HyperLTL and STL-GO encodings on $\mathcal{N}{\times}\mathcal{N}$ grids ($T=4(\mathcal{N}{-}1)+2$).}
\label{tab:hyperltl-scaling}
\setlength{\tabcolsep}{4pt}
\begin{tabular}{@{}cl rrr@{}}
\toprule
$\mathcal{N}$ & Method & |Var.| & |Constr.| & Time (s) \\
\midrule
\multirow{4}{*}{5 $\times$ 5}
 & HyperLTL + SMT           &     38           &   8.6k           &  0.10          \\
 & \textbf{STL-GO + SMT}    & \textbf{395}     & \textbf{888}     & \textbf{0.15}  \\
 & HyperLTL + MIP           &  1.0k            &   8.8k           &  0.20          \\
 & \textbf{STL-GO + MIP}    & \textbf{1.2k}    & \textbf{12.7k}   & \textbf{0.10}  \\
\midrule
\multirow{4}{*}{7 $\times$ 7}
 & HyperLTL + SMT           &     54           &  54.2k           &  0.61          \\
 & \textbf{STL-GO + SMT}    & \textbf{671}     & \textbf{1.4k}    & \textbf{0.34}  \\
 & HyperLTL + MIP           &  2.8k            &  54.7k           &  2.39          \\
 & \textbf{STL-GO + MIP}    & \textbf{3.0k}    & \textbf{65.3k}   & \textbf{0.35}  \\
\midrule
\multirow{4}{*}{10 $\times$ 10}
 & HyperLTL + SMT           &     78           & 355.1k           &  4.26          \\
 & \textbf{STL-GO + SMT}    & \textbf{1.2k}    & \textbf{2.4k}    & \textbf{0.99}  \\
 & HyperLTL + MIP           &  8.2k            & 356.2k           &  4.47          \\
 & \textbf{STL-GO + MIP}    & \textbf{8.5k}    & \textbf{387.5k}  & \textbf{1.96}  \\
\bottomrule
\end{tabular}
\end{table}

\noindent\textbf{Conclusions.}
We presented a synthesis framework for spatio-temporal logical specifications with graph operators on multi-agent systems. The MIP and SMT encodings of STL-GO enable centralized planning over dynamic interaction graphs, with soundness guarantees. Our simulations demonstrate trajectory synthesis under complex STL-GO specifications involving multiple dynamic graphs and role-typed agents.

\noindent \textbf{Limitations and Future Work.} Our soundness guarantees rely on deterministic environment and agent
dynamics. In stochastic settings, the synthesized plans can
only be guaranteed to satisfy the specification with some probability,
and encodings for distributionally
robust formulations are an extension. As we synthesize open-loop
control sequences, embedding
encodings inside a receding-horizon loop or learning
policies that respect STL-GO specifications are open directions. Furthermore, decentralized synthesis under partial
observability, together with decomposition strategies to mitigate the
combinatorial growth in the number of graphs, remains for future
work.
\section{Acknowledgements}
This work was supported by the National Science Foundation under Grant IIS-SLES-2417075 and Lockheed Martin Advanced Technology Laboratories.

\bibliographystyle{IEEEtran}
\bibliography{references}

\appendices
  
\section{Theoretical Results for MIP Encoding}
\label{sec:appendix-theory-milp}
\subsection{Proof of Lemma~\ref{lemma:milp-graph} (MIP, Agent-Local)}
\label{sec:appendix-proof-milp-lemma}
\begin{proof}
We prove the invariant~\eqref{eq:milp-invariant} for every agent-local subformula $\psi$, agent $i$, and time $t$, by structural induction on $\psi$.

\mypara{Dynamics}
Constraint~\eqref{eq:affine-agent-dyn} is an equality in the decision variables; any feasible assignment to $\{\agentstate^i_t, \agentinput^i_t\}_{i,t}$ corresponds to a trajectory of $F$, with~\eqref{eq:state_input_constraints} restricting states and inputs to their admissible sets.

\mypara{Base case ($\psi = \mu_x$)}
The Big-$M$ encoding (Appendix~\ref{sec:appendix-MILP-encoding}) introduces $z_{\mu_x, i, t} \in \{0,1\}$ with $M$ larger than $\sup_{x \in [\agentstate_{\min}, \agentstate_{\max}]} |a^\top x - b|$ and assumes feasible predicate values avoid the separation interval $(-\varepsilon,0)$. It enforces $z_{\mu_x, i, t} = 1 \iff a^\top \agentstate^i_t \ge b$, which is the definition of $(\mas, i, t) \models \mu_x$.

\mypara{Boolean connectives and until operator}
Negation uses $z_{\neg\varphi, i, t} = 1 - z_{\varphi, i, t}$; conjunction and disjunction use standard linear encodings; when $t+b\leq T$, until uses witness-time indicators $\beta^i_{\tau, t}$ enforcing the existence of $\tau \in t \oplus I$ with $z_{\varphi_2, i, \tau} = 1$ and $z_{\varphi_1, i, \tau'} = 1$ for all $\tau' \in [t, \tau)$; when $t+b>T$, its satisfaction variable is fixed to zero~\cite{stl-to-milp1,stl-to-milp2}.

\mypara{Graph operator ($\psi = \Inop_{\graphs, [e_1, e_2]}^{W, \#} \varphi$)}
We establish the invariant for each graph type $\type\in\settypes$, then extend to the modal quantifier.

\emph{(i) Eligibility.} By MIP-encodability of $\gen^{\type}$, the edge
indicator $a_{j,i,t}^{\type}$ exactly represents
$\eta_{j,i}^{\type}(\jointstate_t,\world_t)$ and the weight
$\weights_t^{\type}(j,i)$ is a PWA expression in the MIP variables. With
$M > \max(|w_{\min}|, |w_{\max}|) + \sup |\weights_t^{\type}(j,i)|$ and the
stated separation margin,~\eqref{eq:enc-in1} enforces
$\gamma_{j,i,t}^{\type}=1$ if and only if $(j,i)\in\edge_t^{\type}$ and
$\weights_t^{\type}(j,i)\in W$.

\emph{(ii) Conjunction.}~\eqref{eq:enc-in2} is the standard linear encoding of $y_{j,i,t}^{\varphi,\type} = \gamma_{j,i,t}^{\type} \wedge z_{\varphi, j, t}$; by inductive hypothesis on $\varphi$ at $j$, $y_{j,i,t}^{\varphi,\type} = 1$ iff edge $(j, i)$ is eligible and $(\mas, j, t) \models \varphi$.

\emph{(iii) Cardinality.} The integer count $c_{i, t}^{\Inop, \varphi, \type} = \sum_{j \ne i} y_{j, i, t}^{\varphi, \type}$ is the cardinality of the qualifying neighbor set. Since $c$ is integer-valued, the strict inequalities $c < e_1$ and $c > e_2$ are equivalent to $c \le e_1 - 1$ and $c \ge e_2 + 1$; thus, the unit integer separation in~\eqref{eq:enc-in3}--\eqref{enc-in4} enforces $z_{\psi, i, t}^{\type} = 1 \iff c \in [e_1, e_2]$.

\emph{(iv) Quantification over types.}~\eqref{eq:enc-in5} for $\# = \exists$ and~\eqref{eq:enc-in6} for $\# = \forall$ are the standard disjunction and conjunction encodings over per-type variables.
The outgoing case is obtained by substituting $(j, i)$ with $(i, j)$ in (i)--(iii).

By induction,~\eqref{eq:milp-invariant} holds for every subformula at every $(i, t)$; the encoding asserts $z_{\phi, i, 0} = 1$ for the agent-local root, so feasibility yields $(\mas, i, 0) \models \phi$.
\end{proof}

\subsection{Proof of Theorem~\ref{thm:milp-soundness}}
\label{sec:appendix-proof-milp-theorem}

\begin{proof}
    
We extend~\eqref{eq:milp-invariant} to multi-agent subformulas as
$z_{\psi, t} = 1 \iff (\mas, t) \models \psi$,
by structural induction on the multi-agent grammar.

\mypara{Joint atomic predicate ($\psi = \mu$)}
The Big-$M$ encoding applied to $\mu(\jointstate_t, \world_t)$ enforces $z_{\mu, t} = 1 \iff \mu(\jointstate_t, \world_t)$ by the same argument as the agent-local atomic case.

\mypara{Embedding ($\psi = i.\varphi$)}
The encoding sets $z_{i.\varphi, t} := z_{\varphi, i, t}$, which by Lemma~\ref{lemma:milp-graph} equals the truth of $\varphi$ at $(i, t)$.

\mypara{Boolean connectives and until operator}
Identical in form to the agent-local cases of Lemma~\ref{lemma:milp-graph}, applied to multi-agent variables.

\mypara{Multi-agent quantifiers ($\psi = \EX \varphi, \FA \varphi$)}
The encodings in Section~\ref{sec:quant-specs-milp} enforce $z_{\psi, t} = 1 \iff \bigvee_{i \in \nodes} z_{\varphi, i, t} = 1$ and $z_{\psi, t} = 1 \iff \bigwedge_{i \in \nodes} z_{\varphi, i, t} = 1$. By Lemma~\ref{lemma:milp-graph}, these match the multi-agent semantics.

By induction the invariant holds at the root: $z_{\phi, 0} = 1 \iff (\mas, 0) \models \phi$, and feasibility with $z_{\phi, 0} = 1$ as a constraint yields $(\mas, 0) \models \phi$.
\end{proof}

\section{Theoretical Results of SMT Encoding}
\label{sec:appendix-theory-smt}
\subsection{Proof of Lemma~\ref{lemma:smt-graph} (SMT, Agent-Local)}
\label{sec:appendix-proof-smt-lemma}

\begin{proof}
We prove the invariant~\eqref{eq:smt-invariant} for every agent-local subformula $\psi$, agent $i$, and time $t$, by structural induction on $\psi$. The argument mirrors the proof of Lemma~\ref{lemma:milp-graph}, with Big-$M$ encodings replaced by LRA biconditionals.

\mypara{Dynamics}
Constraint~\eqref{eq:smt-dynamics} is an equality in the SMT variables; any satisfying assignment corresponds to a trajectory of $F$, with~\eqref{eq:smt-state-input-bounds} restricting states and inputs.

\mypara{Base case, Boolean connectives, and until operator}
Atomic predicates use $z^i_{\mu_x, t} = (a^\top \agentstate^i_t - b \ge 0)$. Boolean connectives use the corresponding propositional biconditionals. When $t+b\leq T$, until is unrolled over the bounded horizon as $z^i_{\psi, t} = \bigvee_{\tau = t+a}^{t+b} \big(z^i_{\varphi_2, \tau} \wedge \bigwedge_{k=t}^{\tau - 1} z^i_{\varphi_1, k}\big)$; when $t+b>T$, its satisfaction variable is fixed to $\bot$~\cite{momtaz2023monitoring,prabhakar2018automatic}. Each case transfers~\eqref{eq:smt-invariant} via the inductive hypothesis.

\mypara{Graph operator ($\psi = \Inop_{\graphs, [e_1, e_2]}^{W, \#} \varphi$)}
For each $\type\in\settypes$, SMT-encodability of $\gen^{\type}$ ensures that the edge-existence predicate
$\eta_{j,i}^{\type}(\jointstate_t,\world_t)$ is an LRA+LIA Boolean formula and
the weight $e_{j,i}^{\type}(\jointstate_t,\world_t)$ is an LRA expression.

\emph{(i) Eligibility.}~\eqref{eq:smt-eligible} asserts
$b^{\type}_{j,i,t}\leftrightarrow
(\eta^{\type}_{j,i}\wedge w_{\min}\le e^{\type}\le w_{\max})$,
giving $b^{\type}_{j,i,t}=\top$ if and only if $(j,i)$ is an edge and
$\weights_t^{\type}(j,i)\in W$.

\emph{(ii) Conjunction.}~\eqref{eq:smt-conjunct} asserts $n^{\type}_{j,i,t} = b^{\type}_{j,i,t} \wedge z^j_{\varphi, t}$; by inductive hypothesis, $n^{\type}_{j,i,t} = \top$ iff edge $(j, i)$ is eligible and $(\mas, j, t) \models \varphi$.

\emph{(iii) Cardinality.}~\eqref{eq:smt-cardinality-satisfaction} asserts $z^{i, \type}_{\psi, t} = (e_1 \le c^{\type}_{i,t} \le e_2)$, where $c^{\type}_{i,t}$, defined by~\eqref{eq:smt-cardinality}, counts the qualifying neighbors via $\mathsf{ite}$.

\emph{(iv) Quantification over types.}~\eqref{eq:smt-graph-quant} asserts the disjunctive biconditional for $\#=\exists$; replacing the disjunction with a conjunction yields the biconditional for $\#=\forall$.
The outgoing case is obtained by substituting $(j, i)$ with $(i, j)$ in (i)--(iii).

By induction,~\eqref{eq:smt-invariant} holds for every subformula at every $(i, t)$; the encoding asserts $z^i_{\phi, 0} = \top$ for the agent-local root, so satisfiability yields $(\mas, i, 0) \models \phi$.
\end{proof}

\subsection{Proof of Theorem~\ref{thm:smt-soundness} (SMT, Multi-Agent)}
\label{sec:appendix-proof-smt-theorem}

\begin{proof}
We extend~\eqref{eq:smt-invariant} to multi-agent subformulas as $z_{\psi, t} = \top \iff (\mas, t) \models \psi$, by structural induction on the multi-agent grammar.

\mypara{Joint atomic predicate ($\psi = \mu$)}
The LRA $z_{\mu, t} = \mu(\jointstate_t, \world_t)$ matches the multi-agent atomic semantics.

\mypara{Embedding ($\psi = i.\varphi$)}
The encoding sets $z_{i.\varphi, t} := z^i_{\varphi, t}$, which by Lemma~\ref{lemma:smt-graph} equals the truth of $\varphi$ at $(i, t)$.

\mypara{Boolean connectives and until operator}
Identical in form to the agent-local cases of Lemma~\ref{lemma:smt-graph}, applied to multi-agent variables.

\mypara{Multi-agent quantifiers ($\psi = \EX \varphi, \FA \varphi$)}
The encoding of the multi-agent quantifiers $\EX$ and $\FA$ asserts $z_{\psi, t} = \bigvee_{i \in \nodes} z^i_{\varphi, t}$ and $z_{\psi, t} = \bigwedge_{i \in \nodes} z^i_{\varphi, t}$. By Lemma~\ref{lemma:smt-graph}, these match the multi-agent quantifier semantics.

By induction the invariant holds at the root: $z_{\phi, 0} = \top \iff (\mas, 0) \models \phi$, and satisfiability with $z_{\phi, 0} = \top$ as a constraint yields $(\mas, 0) \models \phi$.
\end{proof}

\section{Full Mixed Integer Encoding for STL-GO}
\label{sec:appendix-MILP-encoding}
\mypara{Predicate Encoding}
Atomic predicates in STL-GO formulas are inequalities of the form
$ \psi := \mu(\sysstate_t^i)$ where $ \mu(\sysstate_t^i) \equiv a^\top \sysstate_t^i - b \ge 0$, representing a geometric or
logical condition on the agent’s state (for instance, being within a
goal region or a communication range).
Each predicate $\psi$ is associated with a binary variable
$z_{\psi,t}^i \in \{0,1\}$ indicating whether $\psi$ holds for agent $i$ at time
$t$. The relationship between the
constraint and the binary indicator is established using the Big-$M$ method:
\[
a^\top \sysstate_t^i - b \ge -M (1 - z_{\psi,t}^i); \quad
a^\top \sysstate_t^i - b \le -\varepsilon + M z_{\psi,t}^i \]
where $M > 0$ is a sufficiently large constant and $\varepsilon > 0$ is a small
positive separation margin. We assume feasible predicate values do not lie in
$(-\varepsilon,0)$. These constraints ensure that
$z_{\psi,t}^i = 1$ if and only if $\psi$ holds.

\mypara{Logical Operators}

Let $\psi = \operatorname{Op}(\varphi_1,\ldots,\varphi_m)$ denote a formula
obtained by applying a logical operator to subformulas
$\varphi_1,\ldots,\varphi_m$.
For each agent $i$ and time $t$, we introduce a binary variable
$z_{\psi,t}^i \in \{0,1\}$ encoding the truth value of $\psi$.
Standard linear encodings are used:
\begin{itemize}\setlength{\itemsep}{1pt}
    \item \textbf{Negation: ($\psi=\neg\varphi$)}
    \(
    z_{\psi,t}^i = 1 - z_{\varphi,t}^i.
    \)
    \item \textbf{Conjunction:}($\psi=\bigwedge_{j=1}^m \varphi_j$)
     \(
    z_{\psi,t}^i \le z_{\varphi_j,t}^i \ \forall j, \ \ 
    z_{\psi,t}^i \ge 1 - m + \sum_{j=1}^m z_{\varphi_j,t}^i.
    \)
    \item \textbf{Disjunction:} ($\psi=\bigvee_{j=1}^m \varphi_j$)
    \(
    z_{\psi,t}^i \ge z_{\varphi_j,t}^i \ \forall j, \ \ 
    z_{\psi,t}^i \le \sum_{j=1}^m z_{\varphi_j,t}^i.
    \)
\end{itemize}
\mypara{Temporal Operators}

Temporal modalities in STL-GO are represented over bounded time horizons using recursive constraints on binary variables. The fundamental operator is \emph{until} ($\Until$); \emph{eventually}
($\Eventually$) and \emph{always} ($\Globally$) are defined as
\(
\Eventually_I \varphi := \top\,\Until_I\,\varphi
\) and
$\Globally_I\varphi := \neg\Eventually_I\neg\varphi$.

\begin{itemize}
\item \textbf{Until.}
For $\psi=\varphi_1\,\Until_{[a,b]}\,\varphi_2$ and $t+b\leq T$,
satisfaction requires that $\varphi_2$ becomes true at some time
$\tau\in[t+a,t+b]$ and that $\varphi_1$ holds at every time in
$[t,\tau)$. Introduce auxiliary variables
$\beta_{\tau,t}^i\in\{0,1\}$, one for each candidate witness time:
\begin{equation}
\begin{gathered}
\beta_{\tau,t}^i \le z_{\varphi_2,\tau}^i,\\
\beta_{\tau,t}^i \le z_{\varphi_1,k}^i
\quad \forall k\in[t,\tau-1],\\
\beta_{\tau,t}^i
\ge z_{\varphi_2,\tau}^i
+\sum_{k=t}^{\tau-1}z_{\varphi_1,k}^i-(\tau-t),
\\
z_{\psi,t}^i \ge \beta_{\tau,t}^i
\quad \forall \tau\in[t+a,t+b],\\
z_{\psi,t}^i \le \sum_{\tau=t+a}^{t+b}\beta_{\tau,t}^i.
\end{gathered}
\end{equation}
For $t+b>T$, the strong bounded-horizon semantics are enforced by
$z_{\psi,t}^i=0$.

    \item \textbf{Eventually.} For $\psi=\Eventually_{[a,b]}\varphi$, apply the derivation $\Eventually_{[a,b]} \varphi := \top\,\Until_{[a,b]}\,\varphi$ and the encoding simplifies since
 $z_{\top,t}^i = 1$ always holds:
For $t+b\leq T$, impose
\begin{equation}
\begin{aligned}
z_{\psi,t}^i &\ge z_{\varphi,\tau}^i
&&\forall \tau\in[t+a,t+b],\\
z_{\psi,t}^i &\le \sum_{\tau=t+a}^{t+b} z_{\varphi,\tau}^i.
\end{aligned}
\end{equation}
For $t+b>T$, impose $z_{\psi,t}^i=0$.
    This ensures $z_{\psi,t}^i=1$ if $\varphi$ holds at least once within $[t+a,t+b]$.
    
    \item \textbf{Globally:}
  For $\psi = \Globally_{[a,b]} \varphi$, apply the derivation $\Globally_{[a,b]} \varphi := \neg\Eventually_{[a,b]} \neg\varphi$. 
For $t+b\leq T$, impose
\begin{equation}
\begin{aligned}
z_{\psi,t}^i &\le z_{\varphi,\tau}^i
&&\forall \tau\in[t+a,t+b],\\
z_{\psi,t}^i &\ge \sum_{\tau=t+a}^{t+b} z_{\varphi,\tau}^i-(b-a).
\end{aligned}
\end{equation}
For $t+b>T$, impose $z_{\psi,t}^i=1$.
    This enforces $z_{\psi,t}^i=1$ only if $\varphi$ holds at all times in $[t+a,t+b]$.
\end{itemize}

\section{Full SMT Encoding for STL-GO}
\label{sec:appendix-SMT-encoding}

\noindent\mypara{Atomic Predicates}
For each atomic predicate $\mu$, we associate a Boolean satisfaction variable 
$z_{\mu,t}^{i} \in \{\mathit{true},\mathit{false}\}$,
to indicate the satisfaction of $\mu$ for agent $i$ at time $t$. 
We encode the semantics of the predicate by asserting the following logical equivalence 
in the Theory of Linear Real Arithmetic (LRA):
\[
z_{\mu,t}^i = \mathit{true} \;\leftrightarrow\; (a^\top x_t^i - b \ge 0).
\]
This constraint couples the discrete Boolean structure with the continuous state vector $x_t^i$,
allowing the solver to reason about system states.

\noindent\mypara{Logical operators}

\begin{itemize}
\item \textbf{Negation.}
We encode a formula $\psi = \neg \varphi$ by imposing the constraint
\begin{equation}
    z_{\psi, t}^{i} = \neg z_{\varphi, t}^{i}
\end{equation}
for each agent $i$ and time $t$. This enforces that the negated formula holds at $(i,t)$ if and only if the inner subformula does not hold at $(i,t)$.

\item \textbf{Conjunction.}
For a formula $\psi = \varphi_1 \wedge \varphi_2$, we add the constraint
\begin{equation}
z_{\psi, t}^{i} = z_{\varphi_1, t}^{i} \wedge z_{\varphi_2, t}^{i}
\end{equation}

\item \textbf{Disjunction.}
Similarly, $\psi = \varphi_1 \vee \varphi_2$ is encoded as
\begin{equation}
z_{\psi, t}^{i} = z_{\varphi_1, t}^{i} \vee z_{\varphi_2, t}^{i}
\end{equation}
\end{itemize}

\noindent\mypara{Temporal operators}
Temporal operators are encoded by imposing constraints on $z_{\psi, t}^{i}$ to satisfy subformulas within the time interval. 
We eliminate temporal quantifiers by unrolling them into finite conjunctions and disjunctions over the bounded time horizon, thereby generating a quantifier-free encoding.


\begin{itemize}
\item \textbf{Until.}
For a formula $\psi = \varphi_1 \, \Until_{[a,b]}\, \varphi_2$,
its discrete-time semantics is captured by the constraint
\begin{equation}
z_{\psi,t}^{i}=
\begin{cases}
\displaystyle\bigvee_{\tau=t+a}^{t+b}
\left(
z_{\varphi_2,\tau}^{i}
\wedge
\displaystyle\bigwedge_{k=t}^{\tau-1}z_{\varphi_1,k}^{i}
\right), & t+b\leq T,\\
\mathit{false}, & t+b>T.
\end{cases}
\end{equation}
This enforces that $\varphi_2$ becomes true at some time $\tau$ within the interval,
and that $\varphi_1$ holds continuously until that time.

\item \textbf{Eventually.}
Consider a formula $\psi = \Eventually_{[a,b]} \varphi$, where $0 \le a \le b$,
the Eventually operator holds at time $t$ if and only if the inner subformula
holds at least once within the time interval. We encode this operator by the constraint
\begin{equation}
z_{\psi,t}^{i}=
\begin{cases}
\displaystyle\bigvee_{\tau=t+a}^{t+b}z_{\varphi,\tau}^{i},
& t+b\leq T,\\
\mathit{false}, & t+b>T.
\end{cases}
\end{equation}

\item \textbf{Always.}
We encode the formula $\psi = \Globally_{[a,b]} \varphi$ by
\begin{equation}
z_{\psi,t}^{i}=
\begin{cases}
\displaystyle\bigwedge_{\tau=t+a}^{t+b}z_{\varphi,\tau}^{i},
& t+b\leq T,\\
\mathit{true}, & t+b>T.
\end{cases}
\end{equation}
This ensures that the Always operator holds at time $t$ exactly when the inner subformula
holds at all time steps within the interval.
\end{itemize}

\section{Comparison to HyperLTL}
\label{app:hyperprops}

The specifications introduced in Section~\ref{sec:preliminaries} are stated in
STL-GO, where agent-level quantifiers such as $\EX$ are evaluated over $\nodes$
\emph{pointwise} at the temporal instant under consideration. The mission also
uses finite conjunctions over the fixed locator, rescuer, and emergency-site
sets, together with predicates derived from the time-varying task graph.
HyperLTL, by contrast, quantifies exclusively over \emph{traces} at the
outermost prefix, so any agent-level quantifier that appears inside a temporal
operator in STL-GO must be \emph{reified}: it must be lifted out of the temporal
scope and re-expressed using outer trace quantifiers, auxiliary atomic
propositions, or finite disjunctions over a fixed agent universe. 
This appendix makes the cost of that reification explicit.

\subsection{Reification conventions}
\label{app:reif:conventions}
We fix the locator set $\mathcal{L}$, rescuer set $\mathcal{R}$, and
emergency-site set $\mathcal{M}$ at design time, and associate one trace
variable with each corresponding entity:
$\pi_\ell$ for $\ell \in \mathcal{L}$, $\pi_r$ for $r \in \mathcal{R}$,
and $\pi_m$ for $m \in \mathcal{M}$. Emergency activation is represented by
$\varphi^{\mathrm{emg},\pi_m}$, while
$\mathsf{task}^{\pi_\ell,\pi_r,\pi_m}$ is true when locator $\ell$ assigns
rescuer $r$ specifically to emergency site $m$. The
predicates $\varphi^{\mathrm{emg}}_m$, $\varphi^{\mathrm{sense}}_m$,
$\varphi^{\mathrm{LL}}_\ell$, $\varphi^{\mathrm{LR}}_\ell$,
$\varphi^{\mathrm{near}}_{r,m}$,
$\varphi^{\mathrm{carry}}_r$, and $\varphi^{\mathrm{atC}}_r$ are tagged
with the trace(s) on which they are evaluated.

We consider two reification strategies. The usual approach
expands every inner agent-level existential into a finite disjunction
over the fixed agent universe, keeping the quantifier prefix
purely universal. The \emph{Skolemized} strategy lifts inner
existentials to outer trace quantifiers, exposing the witness traces
explicitly at the cost of one quantifier alternation per lifted
existential.

\subsection{Reified HyperLTL specifications}
\label{app:reif:specs}

\paragraph{Bounded emergency detection}
The original STL-GO formula contains no inner existential, so the
idiomatic and Skolemized reifications coincide:
\begin{align}
  \label{eq:hyper:detect}
  \phi_{\mathrm{detect}}^{\mathrm{HL}}
  \;:=\;
  &\forall \pi_m.\;
  \Globally\Big(
    \varphi^{\mathrm{emg},\pi_m}
    \Rightarrow \notag \\
  &\quad
    \Eventually_{[0,T_{\mathrm{det}}]}\,
    \varphi^{\mathrm{sense},\pi_m}
  \Big).
\end{align}
The finite conjunction $\bigwedge_{m \in \mathcal{M}}$ becomes the universal
trace quantifier $\forall \pi_m$.

\paragraph{Detection-to-relay}
The inner existential $\exists \ell \in \mathcal{L}$ must be reified.
The idiomatic form expands it as a disjunction over $\mathcal{L}$:
\begin{align}
  \label{eq:hyper:relay:idio}
  \phi_{\mathrm{relay}}^{\mathrm{idio}}
  \;:=\;
  &\forall \pi_m.\;
  \Globally\Big(
    \varphi^{\mathrm{emg},\pi_m}\wedge
    \varphi^{\mathrm{sense},\pi_m}
    \Rightarrow \notag \\
  &\quad
    \Eventually_{[0,T_{\mathrm{relay}}]}
    \bigvee_{\ell \in \mathcal{L}}
    (\varphi^{\mathrm{LL},\pi_\ell} \vee \varphi^{\mathrm{LR},\pi_\ell})
  \Big).
\end{align}
The Skolemized form lifts the locator existential to the outer prefix,
introducing one quantifier alternation:
\begin{align}
  \label{eq:hyper:relay:skol}
  \phi_{\mathrm{relay}}^{\mathrm{Skol}}
  \;:=\;
  &\forall \pi_m.\; \exists \pi_\ell.\;
  \Globally\Big(
    \varphi^{\mathrm{emg},\pi_m}\wedge
    \varphi^{\mathrm{sense},\pi_m}
    \Rightarrow \notag \\
  &\quad
    \Eventually_{[0,T_{\mathrm{relay}}]}
    (\varphi^{\mathrm{LL},\pi_\ell} \vee \varphi^{\mathrm{LR},\pi_\ell})
  \Big).
\end{align}
The two reifications are not semantically equivalent: the original
STL-GO formula admits a witness locator that may depend on both $m$ and
the temporal instant $t$, whereas \eqref{eq:hyper:relay:skol} forces a
single witness locator trace per emergency for the entire horizon, and
\eqref{eq:hyper:relay:idio} is faithful only because the disjunction is
re-evaluated at each instant. Capturing the original semantics exactly
would require either the disjunctive form
\eqref{eq:hyper:relay:idio} or a Skolemization in which the witness
varies with $t$, which HyperLTL cannot express without further auxiliary
machinery.

\paragraph{Detection-to-assignment}
The reifications follow the same pattern. Idiomatic:
\begin{align}
  \label{eq:hyper:assign:idio}
  \phi_{\mathrm{assign}}^{\mathrm{idio}}
  \;:=\;
  &\forall \pi_m.\;
  \Globally\Big(
    \varphi^{\mathrm{emg},\pi_m}\wedge
    \varphi^{\mathrm{sense},\pi_m}
    \Rightarrow \notag \\
  &\quad
    \Eventually_{[0,T_{\mathrm{assign}}]}
    \bigvee_{\ell \in \mathcal{L}}
    \bigvee_{r \in \mathcal{R}}
    \mathsf{task}^{\pi_\ell,\pi_r,\pi_m}
  \Big).
\end{align}
Skolemized:
\begin{align}
  \label{eq:hyper:assign:skol}
  \phi_{\mathrm{assign}}^{\mathrm{Skol}}
  \;:=\;
  &\forall \pi_m.\; \exists \pi_\ell.\;\exists \pi_r.\;
  \Globally\Big(
    \varphi^{\mathrm{emg},\pi_m}\wedge
    \varphi^{\mathrm{sense},\pi_m}
    \Rightarrow \notag \\
  &\quad
    \Eventually_{[0,T_{\mathrm{assign}}]}\,
    \mathsf{task}^{\pi_\ell,\pi_r,\pi_m}
  \Big).
\end{align}

\paragraph{Rescuer response time}
The fixed conjunctions over locators, rescuers, and emergency sites become
universal trace quantifiers. The emergency-specific assignment predicate
preserves the association between the assigned rescuer and site:
\begin{align}
  \label{eq:hyper:reach}
  \phi_{\mathrm{reach}}^{\mathrm{HL}}
  \;:=\;
  &\forall\pi_\ell.\,\forall\pi_r.\,\forall\pi_m.\;
  \Globally\Big(
    \mathsf{task}^{\pi_\ell,\pi_r,\pi_m}
    \wedge \varphi^{\mathrm{emg},\pi_m}
    \Rightarrow \notag \\
  &\quad
    \Eventually_{[0,T_{\mathrm{reach}}]}
    \varphi^{\mathrm{near},\pi_r,\pi_m}
  \Big).
\end{align}

\paragraph{Rescue and deliver}
The same universal prefix gives the direct reification:
\begin{align}
  \label{eq:hyper:deliver}
  \phi_{\mathrm{deliver}}^{\mathrm{HL}}
  \;:=\;
  &\forall\pi_\ell.\,\forall\pi_r.\,\forall\pi_m.\;
  \Globally\bigg(
    \mathsf{task}^{\pi_\ell,\pi_r,\pi_m}
    \wedge \varphi^{\mathrm{near},\pi_r,\pi_m}
    \Rightarrow \notag \\
  &\quad
      \Eventually_{[0,T_{\mathrm{deliver}}]}\!
      (\varphi^{\mathrm{carry},\pi_r} \wedge \varphi^{\mathrm{atC},\pi_r})
  \bigg).
\end{align}

\subsection{Sources of the succinctness loss}
\label{app:reif:cost}

Two principal structural costs are imposed by reification, and both are visible
in the formulas above.

First, every pointwise agent choice in STL-GO---whether expressed by $\EX$
over $\nodes$ or by a finite role-restricted disjunction---becomes either a
finite disjunction whose size scales with the corresponding agent universe
or an outer trace quantifier that raises the alternation
depth. In the disjunctive form, the propositional matrix grows as
$\Theta(|\mathcal{L}|)$ for $\phi_{\mathrm{relay}}$ and
$\Theta(|\mathcal{L}|\,|\mathcal{R}|)$ for $\phi_{\mathrm{assign}}$.
The role-restricted choices in both formulas remain inside the temporal
operator, preserving their pointwise dependence on the time and emergency;
the assignment formula also preserves the emergency-specific locator--rescuer
association. In the
Skolemized form, the matrix remains compact but the prefix acquires a
universal--existential alternation. By contrast, the fixed conjunctions
in $\phi_{\mathrm{reach}}$ and $\phi_{\mathrm{deliver}}$ translate to
universal trace prefixes and introduce no alternation.

Second, joint predicates such as $\varphi^{\mathrm{near}}_{r,m}$ and
$\varphi^{\mathrm{task}}_{\ell,r,m}$ that
depend on the states of multiple agents map naturally onto STL-GO via
predicate functions over agent-state tuples at the outer level. In
HyperLTL they require atomic propositions tagged with multiple path
variables, written above as $\varphi^{\mathrm{near},\pi_r,\pi_m}$ and
$\mathsf{task}^{\pi_\ell,\pi_r,\pi_m}$. This is a non-standard extension;
in strict HyperLTL one must either replicate each predicate across the
participating traces with consistency constraints or
precompute it as a Boolean signal on a designated system trace, in either
case adding further auxiliary atomic propositions.

Taken together, the reification yields HyperLTL formulas whose
propositional matrix scales linearly with the locator universe for relay and
with the locator--rescuer product for assignment (idiomatic form), or whose quantifier prefixes acquire an
alternation (Skolemized form). The remaining formulas stay universally
quantified but retain the multi-trace predicate overhead.

\subsection{Experimental comparison with HyperLTL}
\label{sec:app-exp-hyperltl}
We evaluate STL-GO against a HyperLTL baseline on the wildfire-rescue
grid-world benchmark of HypRL~\cite{hsu2025hyprl}. The environment is an
$\mathcal{N} \times \mathcal{N}$ grid. Two heterogeneous agents, a firefighter
(FF) and a medical responder (Med), are deployed from a shared initial
cell in the top-left corner. Let $\mathcal F_{\mathcal N}$ and
$\mathcal V_{\mathcal N}$ denote, respectively, the fire and victim cells for
an $\mathcal N\times\mathcal N$ instance; their cardinalities scale with
$\mathcal N$, with fire cells placed along the anti-diagonal and victim cells
along the middle row.

The mission is governed by four sub-specifications. Objectives
\textbf{O1} and \textbf{O2} require FF to eventually visit every fire
zone (extinguishing it) and Med to eventually visit every victim cell respectively.
Coordination constraint \textbf{C1} enforces that the agents remain
within Manhattan communication range $\mathcal N-1$ at all times. Safety
constraint \textbf{C2} imposes a temporal precedence: Med may not enter
any fire zone until FF has already visited that cell. The HyperLTL
specification uses a $\forall \pi_{FF}.\, \exists \pi_{Med}$ prefix and
joint multi-trace propositions: 

\begin{align*}
\phi_{\mathrm{Rescue}}
&\;:=\;
\forall \pi_{FF}.\exists \pi_{Med}.\,
(
\psi_{\mathrm{fire}}
\wedge
\psi_{\mathrm{save}}
\wedge
\psi_{\mathrm{dist}}
\wedge
\psi_{\mathrm{safe}}
)
\\
\textbf{O1: }\;
\psi_{\mathrm{fire}} &\;:=\;
\bigwedge_{q\in\mathcal F_{\mathcal N}} \Eventually(q^{\pi_{FF}}) \\
\textbf{O2: }\;
\psi_{\mathrm{save}} &\;:=\;
\bigwedge_{q\in\mathcal V_{\mathcal N}} \Eventually(q^{\pi_{Med}})\\
\textbf{C1: }\;
\psi_{\mathrm{dist}} &\;:=\;
\Globally \left( \left\|
\mathrm{Location}^{\pi_{FF}} - \mathrm{Location}^{\pi_{Med}} \right\|_1 < \mathcal N
\right)\\
\textbf{C2: }\;
\psi_{\mathrm{safe}} &\;:=\;
\bigwedge_{q\in\mathcal F_{\mathcal N}}
(\neg q^{\pi_{Med}} \;\Until\; q^{\pi_{FF}})
\end{align*}
The equivalent STL-GO specification captures the same mission using
STL-GO operators: a shared fire-status state variable collapses the
per-cell until of \textbf{C2} into the constraint
$\mathrm{loc}[\mathrm{Med},q,t]+\mathrm{fire}[q,t]\le 1$ for each
$q\in\mathcal F_{\mathcal N}$, and the singleton communication-graph
collection $\{\graph^c_t\}$ with the operator
$\Outop^{\exists}_{\{\graph^c\}, [1, \infty)} \top$ at FF
enforces \textbf{C1} via a single distance constraint per timestep.



\end{document}